\documentclass[a4paper,fleqn]{cas-dc}
\usepackage[authoryear]{natbib} 

\usepackage{threeparttable}
\usepackage{flushend}
\usepackage{graphicx}

\usepackage{subcaption}

\def\tsc#1{\csdef{#1}{\textsc{\lowercase{#1}}\xspace}}
\tsc{WGM}
\tsc{QE}

\newtheorem{thm}{Theorem}
\newtheorem{rmk}{Remark}
\newtheorem{proof}{Proof}
\usepackage{amsmath}
\usepackage{upgreek}
\usepackage{appendix}
\usepackage{adjustbox}
\usepackage[pagewise]{lineno}
\usepackage{flushend}
\usepackage{float}
\usepackage{graphicx}

\begin{document}

\let\WriteBookmarks\relax
\def\floatpagepagefraction{1}
\def\textpagefraction{.001}
\let\printorcid\relax

\makeatletter
\renewcommand{\fnum@figure}{Fig. \thefigure.\@gobble}
\makeatother

\shorttitle{}    

\shortauthors{}  


\title[mode = title]{
AEOS: \underline{A}ctive \underline{E}nvironment-aware \underline{O}ptimal \underline{S}canning Control for UAV LiDAR-Inertial Odometry in Complex Scenes}

\tnotetext[1]{This research is supported by NTUitive Gap Fund (NGF-2025-17-006) and the National Research Foundation, Singapore, under its Medium-Sized Center for Advanced Robotics Technology Innovation (CARTIN).}

\author[1]{Jianping Li}
\ead{jianping.li@ntu.edu.sg}
\credit{Conceptualization of this study, Methodology, Writing review, Original draft, Funding acquisition}
\affiliation[1]{organization={School of Electrical and Electronic Engineering},
            addressline={Nanyang Technological University}, 
           city={Singapore},
           postcode={639798}, 
            country={Singapore}}

\author[1]{Xinhang Xu}
\ead{xu0021ng@e.ntu.edu.sg}
\credit{Methodology, Experiment}

\author[1]{Zhongyuan Liu}
\ead{zliu051@e.ntu.edu.sg}
\credit{Methodology, Experiment}

\author[1]{Shenghai Yuan}
\ead{shyuan@ntu.edu.sg}
\credit{Methodology, Experiment}

\author[2]{Muqing Cao}
\ead{caom0006@e.ntu.edu.sg}
\credit{Methodology, Experiment}
\affiliation[2]{organization={Robotics Institute},
            addressline={Carnegie Mellon University}, 
           city={Pittsburgh},
           postcode={15213}, 
            country={USA}}

\author[1]{Lihua Xie}
\ead{elhxie@ntu.edu.sg}
\credit{Conceptualization of this study, Project administration, Funding acquisition}




\begin{abstract}
LiDAR-based 3D perception and localization on unmanned aerial vehicles (UAVs) are fundamentally limited by the narrow field of view (FoV) of compact LiDAR sensors and the payload constraints that preclude multi-sensor configurations. Traditional motorized scanning systems with fixed-speed rotations lack scene awareness and task-level adaptability, leading to degraded odometry and mapping performance in complex, occluded environments.
Inspired by the active sensing behavior of owls, we propose AEOS (Active Environment-aware Optimal Scanning), a biologically inspired and computationally efficient framework for adaptive LiDAR control in UAV-based LiDAR-Inertial Odometry (LIO). AEOS combines model predictive control (MPC) and reinforcement learning (RL) in a hybrid architecture: an analytical uncertainty model predicts future pose observability for exploitation, while a lightweight neural network learns an implicit cost map from panoramic depth representations to guide exploration.
To support scalable training and generalization, we develop a point cloud-based simulation environment with real-world LiDAR maps across diverse scenes, enabling sim-to-real transfer. Extensive experiments in both simulation and real-world environments demonstrate that AEOS significantly improves odometry accuracy compared to fixed-rate, optimization-only, and fully learned baselines, while maintaining real-time performance under onboard computational constraints. The project page can be found at \url{https://kafeiyin00.github.io/AEOS/}.

\end{abstract}

\begin{keywords}
LiDAR\sep Sensor Control\sep UAV\sep LiDAR-Inertial Odometry\sep Reinforcement Learning
\end{keywords}

\maketitle

\begin{figure*}[]
\includegraphics[width=\textwidth]{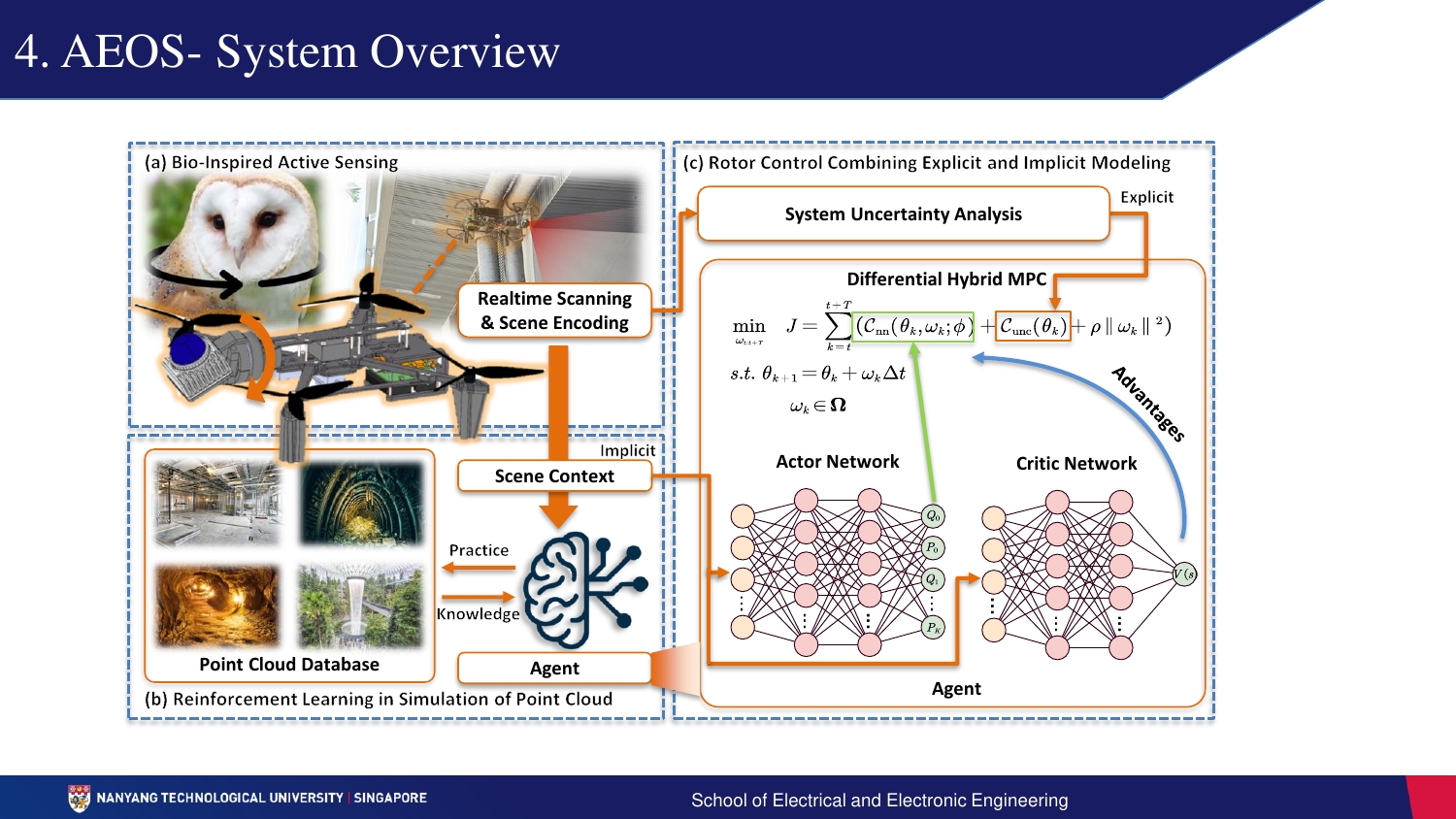}
    \captionof{figure}{Overview of the proposed Active Environment-aware Optimal Scanning Control (AEOS). (a) We proposed a bio-inspired active panoramic sensing UAV system for complex scenes. (b) The active sensing agent practices and learns knowledge in a point cloud-based simulation environment via reinforcement learning. (c) The proposed agent is a differentiable hybrid Model Predictive Control (MPC) considering implicit information from scene context and explicit information from uncertainty analysis.}
    \label{fig:abstract}
\end{figure*}

\section{Introduction}

Autonomous 3D perception in unstructured and occluded environments, such as underground infrastructure \citep{ebadi2023present, li2023whu}, construction sites \citep{chen2022robot, li2024hcto}, or forested areas \citep{li20193d,proudman2022towards}, is a critical enabler for robotic applications in inspection, mapping, monitoring, and search and rescue \citep{wang2023development}. Among various sensing modalities, LiDAR has emerged as a reliable choice for high-precision 3D perception in these complex scenarios. However, when deployed on unmanned aerial vehicles (UAVs), LiDAR-based sensing systems face two fundamental constraints: the inherently narrow field of view (FoV) of compact LiDAR sensors and the strict payload and power limits that preclude the use of multiple LiDAR sensors \citep{chen2023self,diels2022optimal}. These limitations often lead to severe degradation in the accuracy and completeness of LiDAR-inertial odometry (LIO), particularly in complex environments with occlusions or sparse features \citep{chen2024design, kaul2016continuous}.

Recent advances in LiDAR-based perception have explored a progression of strategies to overcome field-of-view (FoV) limitations and improve odometry performance in complex environments. Early efforts focused on motorized scanning mechanisms, such as rotating LiDAR heads or actuated platforms~\citep{kaul2016continuous, alismail2015automatic}, which passively increase spatial coverage but lack the ability to adapt to scene complexity or task-specific demands~\citep{chen2024design, cui2024alphalidar}. To improve adaptability, optimization-based active control methods, typically formulated as model predictive control (MPC) problems, have been proposed to modulate scanning parameters based on estimated uncertainty or information gain~\citep{li2025ua, leung2006planning}. While offering more flexible viewpoint selection, these methods often require handcrafted cost functions and are sensitive to local minima, limiting their robustness in diverse environments.
More recently, learning-based control strategies, particularly those based on deep reinforcement learning (RL), have shown promise in enabling adaptive, scene-aware scanning behaviors learned directly from interaction data~\citep{romero2024actor, morgan2021model}. However, such approaches remain computationally intensive and are difficult to deploy on resource-constrained UAVs due to their reliance on high-dimensional point cloud inputs and large neural networks~\citep{zhang2025armor, ling2023efficacy, cao2024learning}. Additionally, the lack of interpretability in end-to-end learned controllers makes it difficult to enforce safety constraints or ensure generalization across unseen scenes.
Together, these limitations underscore a fundamental challenge: how to design an active LiDAR control framework that is adaptive, computationally lightweight, and physically interpretable, while maintaining strong generalization across diverse and occluded environments.

To overcome these limitations, we introduce Active Environment-aware Optimal Scanning (AEOS), a biologically inspired and hybrid control framework for active LiDAR perception on UAV platforms. AEOS takes inspiration from the adaptive gaze mechanisms observed in owls, which continuously reorient their head to maximize visual awareness under physical constraints. Analogously, AEOS dynamically modulates the LiDAR’s scanning direction and velocity in response to environmental structure and motion-induced uncertainty. 
A key insight in AEOS is the fusion of explicit and implicit modeling to balance exploitation and exploration—two competing objectives in active SLAM. Exploitation refers to focusing sensing resources on task-relevant or feature-rich regions to maximize estimation accuracy, while exploration aims to acquire information about under-observed or uncertain areas to reduce long-term uncertainty \citep{wang2023active}. In AEOS, we use an analytical model to predict pose uncertainty from system dynamics and sensor coverage, which guides the MPC to optimize control for exploitation in a task-aware and interpretable manner. In parallel, a lightweight neural network learns a cost map from local point clouds to capture high-dimensional scene features that support exploration and are hard to model analytically. This hybrid design significantly reduces computational burden compared to pure end-to-end learning, while enhancing generalization across diverse scenes.

To support robust and scalable training, we develop a high-fidelity, point cloud-based simulation platform with real-to-sim-to-real transfer capability across diverse complex scenes: urban, tunnel, and forest scenarios. AEOS is trained and validated within this environment and deployed on a resource-constrained UAV platform with a lightweight motorized LiDAR. 
In summary, we make the following contributions:

(1) We develop a compact, biologically inspired motorized LiDAR control strategy that endows UAVs with panoramic and self-adaptive scanning capabilities, enabling task- and context-aware FoV modulation without exceeding onboard payload or power constraints.

(2) We propose a hybrid RL-MPC framework that fuses explicit and implicit modeling to balance exploitation and exploration in active LiDAR control. An analytical uncertainty model guides MPC to focus sensing actions on rich-feature regions (exploitation), while a neural network implicitly learns a cost map from local point clouds to promote coverage of uncertain areas (exploration). This architecture preserves the interpretability and task-awareness of model-based control while leveraging the adaptability and high-dimensional feature learning of reinforcement learning, enabling improved generalization and lower computational cost compared to fully end-to-end policies.

(3) We develop a high-fidelity, point cloud-based RL simulation environment tailored for active LiDAR control, supporting physics-consistent UAV dynamics, sensor modeling, and diverse real-world scenarios. This environment enables efficient training, evaluation, and sim-to-real transfer of adaptive LiDAR scanning policies.

(4) We validate the proposed system through extensive experiments in both simulated and real-world environments, demonstrating consistent improvements in LiDAR-inertial odometry accuracy, mapping completeness, and robustness over fixed-speed and purely optimization-based baselines.

The remainder of this paper is organized as follows. Section~\ref{sec:related_work} reviews related work in motorized LiDAR systems, active sensor control, and data-driven scanning strategies. Section~\ref{sec:hardware} presents the hardware design for AEOS-Drone. Section~\ref{sec:method} details the proposed hybrid RL-MPC control framework. Section~\ref{sec:simulation} introduces the point cloud-based simulation platform and training pipeline. Section~\ref{sec:experiments} reports experimental results in both simulated and real-world scenarios. Finally, Section~\ref{sec:conclusion} concludes the paper and discusses future directions.

\section{Related Works} \label{sec:related_work}
LiDAR-based active perception and positioning have been developed along a clear trajectory in the past decade: from mechanically expanding sensor coverage \citep{alismail2015automatic}, to introducing task-aware active control \citep{shi2023real,cui2024alphalidar,li2025ua}, and ultimately toward data-driven strategies that offer greater adaptability \citep{cao2023trust,bartolomei2021semantic}. Early efforts focused on motorized scanning systems to overcome the limited field of view (FoV) inherent in single LiDAR sensors. These solutions improved coverage but remained largely passive and uniform in their scanning behavior \citep{kaul2016continuous, alismail2015automatic}. To improve task-level performance, active sensor control approaches are developed to adapt sensor orientation based on task objectives or environmental cues. However, these methods often rely on hand-tuned models or experience, and are prone to suboptimal local decisions \citep{li2025ua,chen2024design}. More recently, the emergence of data-driven control strategies, particularly those based on reinforcement learning (RL), has offered the promise of adaptive, scene-aware perception, though often at the cost of increased complexity and reduced interpretability \citep{cao2023trust,bartolomei2021semantic,romero2024actor}.
In the following, we review this progression in detail, highlighting the capabilities and limitations that motivate our proposed active environment-aware optimal scanning control for UAV system.

\subsection{Motorized Scanning System}
To overcome the inherent field-of-view (FoV) limitations of fixed-mounted LiDAR sensors, motorized scanning systems have been introduced to enable dynamic viewpoint adjustment and panoramic sensing. Early implementations, primarily on ground-based platforms, leveraged multi-degree-of-freedom mechanisms such as rotating heads or actuated gimbals to perform dense 3D scanning and improve SLAM robustness in occluded or large-scale environments~\citep{alismail2015automatic, zhang2014loam}. These systems operated under relaxed payload and power constraints, making high-frequency, omnidirectional scanning feasible. More recent efforts have extended motorized LiDAR to mobile platforms with tighter constraints, such as legged robots~\citep{li2025limo}.

However, transferring these designs to UAVs introduces substantial challenges due to strict size, weight, and power limitations. Most aerial platforms rely on passive or mechanically constrained solutions to improve LiDAR coverage, such as free-spinning mounts~\citep{chen2023self} or MEMS-based optical deflection units~\citep{chen2024design}, that offer limited or no control over scanning behavior. Critically, these systems lack the ability to actively direct sensing based on scene geometry, motion state, or specific task. 

Moreover, current designs are predominantly engineered for geometric completeness, without drawing from bio-inspired principles of active perception. In contrast, many animals, such as owls, dynamically reorient their head and gaze to resolve uncertainty and focus on salient regions, thereby achieving efficient, context-driven sensing under tight physical constraints. These biological strategies, emphasizing adaptive, efficient, and purposeful sensing, have rarely been embodied in LiDAR control for UAVs.

This lack of scene-adaptive and goal-directed scanning highlights a critical gap: a lightweight, computationally efficient, and behaviorally intelligent LiDAR control system that can adapt its sensing strategy on-the-fly remains largely unexplored in the context of autonomous aerial robotics.

\subsection{Active Sensor Control for Sensing System}
While motorized scanning systems extend LiDAR FoV via mechanical means, they typically lack scene-awareness and task adaptability, especially in maintaining reliable SLAM under challenging conditions~\citep{chen2024lidar, ramezani2022wildcat}. To bridge this gap, active sensor control strategies have been explored, adjusting sensor poses or orientations based on real-time perception feedback.

In visual SLAM, a rich body of work has investigated view planning guided by information-theoretic metrics to achieve better SLAM accuracy. For example, a continuous view planning method is proposed based on Fisher information to select camera viewpoints that maximize feature richness and localization accuracy in outdoor environments \citep{wang2023active}. iRotate introduces a rotating camera approach for omnidirectional robots, which decoupled platform motion from camera actuation, enabling active orientation control to improve map quality and reduce travel distance \citep{bonetto2022irotate}. Additionally, ExplORB-SLAM \citep{placed2022explorb} uses a pose-graph topology to compute utility efficiently for active viewpoint selection, demonstrating improved uncertainty reduction during exploration.

Similar principles have been extended to LiDAR-based systems, where active control is used to improve odometry or mapping quality by explicitly adjusting scan direction, angular velocity, or sensor pose. According to the recent survey by ~\citet{placed2023survey}, LiDAR-based active SLAM methods share the common goal of maximizing long-term information gain or minimizing pose uncertainty. These methods can be categorized based on how utility is defined, ranging from information-theoretic criteria (e.g., entropy reduction, mutual information) to task-aware objectives such as localization robustness \citep{stachniss2004exploration}, map coverage \citep{bai2024graph}, or motion efficiency \citep{li2020high}. A common theme is the tight integration between perception objectives and planning under system constraints, such as sensor actuation limits and real-time onboard computation \citep{kantaros2019asymptotically}. For example, \citet{mihalik2022new} proposed an entropy-based 2D grid LiDAR SLAM strategy that dynamically reorients the scanner to reduce map uncertainty. Similarly, \citet{leung2006planning} introduces SPLAM, one of the earliest MPC-based LiDAR active SLAM frameworks, which integrates pose uncertainty into a receding-horizon controller for planning informative trajectories. More recently, efforts like UA-MPC~\citep{li2025ua} have shown that incorporating ray-traced LiDAR observability models into MPC frameworks can enable uncertainty-aware scan speed modulation on handheld devices, improving LIO robustness. Yet, these optimization-based approaches still depend on carefully tuned cost terms and exhibit limited task adaptability.
These methods often assume access to a global belief or map structure and rely heavily on handcrafted utility functions, which limits their generalization to complex scenes.

\subsection{Data-Driven Control using Reinforcement Learning}

To address the limitations of handcrafted and optimization-based controllers, recent research has increasingly explored data-driven control strategies, particularly those based on deep reinforcement learning (RL), for active sensing and motion planning. RL-based policies are capable of learning scene-aware and adaptive behaviors directly from interaction data, making them particularly suitable for long-horizon tasks where system dynamics or perception performance are difficult to model explicitly~\citep{romero2024actor, morgan2021model}. In particular, heuristic formulations based on information gain, while commonly used in active SLAM, often lead to locally optimal decisions that fail to account for long-term positioning uncertainty reduction.

In the context of active perception, several works have employed policy gradient methods or actor–critic frameworks to select sensor poses or motion primitives based on information gain or task-related reward~\citep{chaplot2020learning}. Others have explored hierarchical RL or representation learning from point clouds to support scalable policy training~\citep{xu2025flying}. Hybrid approaches that combine RL with structured controllers such as model predictive control (MPC) have also gained traction. For instance, learned components can be used to tune MPC cost weights~\citep{cao2024learning} or to generate trajectory parameters within an optimization loop~\citep{romero2024actor}.

Despite these advances, direct application of RL to UAV-based LiDAR control remains highly challenging. End-to-end RL policies trained on raw point clouds or occupancy maps tend to require deep networks with high-dimensional input encoders and recurrent modules \citep{ling2023efficacy, zhang2025armor}, leading to an excessive computational burden that exceeds the capabilities of lightweight onboard processors typically available on UAVs. In addition, such models often exhibit poor sample efficiency and limited generalization across environments, particularly in safety-critical domains like aerial robotics where failure is unacceptable \citep{xu2025navrl}. Moreover, the lack of interpretability in fully learned control pipelines makes it difficult to diagnose failures or enforce constraints.

Our work builds upon these insights by proposing a hybrid RL-MPC control framework that integrates the structure of analytical uncertainty modeling with the flexibility of learned cost maps. This design enables efficient, interpretable, and adaptive control for motorized LiDAR in UAV platforms, addressing key limitations observed in both purely model-based and purely learning-based prior works.

\section{Hardware Design of AEOS-Drone} \label{sec:hardware}

\begin{figure}[]
    \centering
    \includegraphics[width=0.9\linewidth]{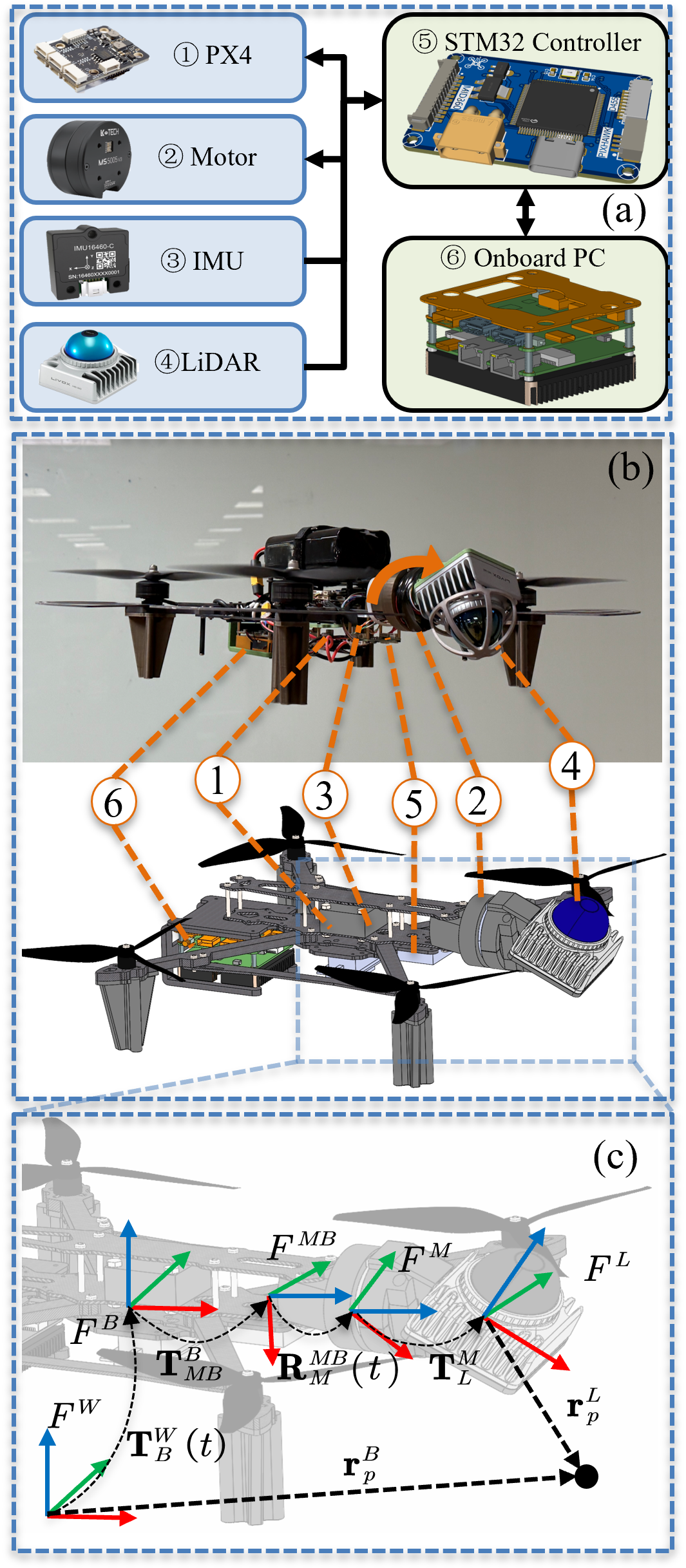}
    \caption{Hardware design of the proposed Active Environment-aware Optimal Scanning (AEOS)-Drone. (a) Sensor synchronization and embedded system architecture. (b) Bio-inspired UAV hardware design and related coordination systems. (c) Coordination systems related to the AEOS-Drone.}
    \label{fig:hardware}
\end{figure}

\subsection{Sensor Synchronization and Embedded System Architecture}
To enable active LiDAR control and accurate sensor fusion, the AEOS-Drone integrates a strict electric time-synchronized multi-sensor system, as shown in Fig. \ref{fig:hardware} (a). The electrical architecture consists of a PX4 flight controller, a motor, an IMU, a LiDAR sensor, a STM32 microcontroller, and an onboard computing unit.

The STM32 microcontroller serves as the central timing and coordination hub. It collects encoder feedback from the rotor, IMU and LiDAR data streams, and synchronizes them with the PX4 state estimator through a unified timestamp protocol. This ensures that all sensor measurements are temporally aligned for consistent state estimation and control.

Data are transmitted to the onboard PC via USB. The onboard PC executes the AEOS control policy, receiving synchronized sensor observations, predicting motor commands, and modulating the LiDAR’s scanning behavior in real time. This architecture enables low-latency feedback, reliable data fusion, and precise control under the computational constraints of aerial platforms.

\subsection{ Bio-Inspired UAV Mechanical Design}

The mechanical design of the AEOS-Drone, shown in Fig. \ref{fig:hardware} (b), draws inspiration from the active gaze control of owls, which adjust their head orientation to focus on salient areas without altering body posture. Similarly, the AEOS platform employs a front-mounted motorized LiDAR module that enables independent adjustment of the sensor’s scanning direction.

The LiDAR unit is mounted on a compact, high-precision rotor, allowing it to rotate about its axis to achieve panoramic perception. This rotation is controlled independently of the UAV's flight dynamics, enabling scene-aware scanning modulation without disturbing the UAV's stability. The design supports variable-speed bidirectional actuation, making it possible to prioritize task-relevant or uncertain regions in real time. The UAV frame is engineered to be lightweight and well-balanced. This bio-inspired mechanical configuration allows the AEOS system to achieve adaptive sensing while preserving flight efficiency and platform robustness in complex environments.

\subsection{ Coordination Systems and Calibration}
To ensure accurate transformation between sensor data and global pose estimates, AEOS defines a set of coordinate frames, as illustrated in Fig.~\ref{fig:hardware} (c). The system includes the world frame ($F^W$), UAV body frame ($F^B$), motor base frame ($F^{MB}$, fixed with UAV body), rotor frame ($F^M$, rotating along Z axis of motor base frame), and LiDAR sensor frame ($F^L$). These frames are linked via a sequence of time-varying and static transformations to describe the spatial configuration of the LiDAR sensor relative to the UAV and the world. The transformation $\mathbf{T}^A_B = [\mathbf{R}^A_B, \mathbf{r}^A_B]$ denotes the rigid-body pose of frame $B$ with respect to frame $A$, where $\mathbf{R}^A_B \in \mathbb{SO}(3)$ is the rotation matrix and $\mathbf{r}^A_B \in \mathbb{R}^3$ is the translation vector. We denote a point observed by the LiDAR in the sensor frame $\mathbf{r}^L_p \in \mathbb{R}^3$, which is represented as $\mathbf{r}^{W}_p \in \mathbb{R}^3$ in the world frame. The transformation between the $\mathbf{r}^{L}_p$ and $\mathbf{r}^{W}_p$ is written as follow:

\begin{equation}
    \mathbf{r}^{W}_p = \mathbf{R}^W_B(t)
    \left( \mathbf{R}^B_{MB}
    \left(\mathbf{R}^{B}_M(t)\left(\mathbf{R}^M_L \mathbf{r}^{L}_p + \mathbf{r}^M_L \right) \right) + \mathbf{r}^B_{MB}
    \right)
    +\mathbf{r}^W_B(t), \label{eq:coordinate_projection}
\end{equation}

Precise extrinsic calibration between these frames is critical for reliable pose estimation and control. Static transformations, such as between the LiDAR and motor frames ($\mathbf{T}_M^L$), and between the motor base and UAV body ($\mathbf{T}_B^{MB}$), are obtained through one-time mechanical alignment and offline calibration. More specifically, the LiDAR-motor calibration for $\mathbf{T}_M^L$ is calibrated using Limo-calib \citep{li2025limo}. Then, UAV body-motor calibration for $\mathbf{T}^B_{MB}$ is conducted using LI-Calib \citep{lv2020targetless}. The time-varying rotation $\mathbf{R}_{MB}^M(t)$ is continuously updated from the rotor encoder via the STM32 controller, and fed into the control and odometry modules in real time. The LIO algorithm will estimate the body transformation $\mathbf{T}^W_B (t)$ with respect to the world frame. In the following section, we will elaborate on how to control the time-varying rotation $\mathbf{R}_{MB}^M(t)$ and estimate the system states.

\section{Our Approach: AEOS}
\label{sec:method}


\subsection{Problem Formulation and System Overview}
We consider the problem of active LiDAR viewpoint control on a UAV platform, where the objective is to dynamically modulate the LiDAR scanning direction to improve localization and mapping performance in complex environments. The goal of the control system is two-fold: (1) to actively direct the LiDAR scanning pattern to reduce future pose uncertainty, thereby improving LiDAR-inertial odometry (LIO); and (2) to explore unobserved or high-uncertainty regions to improve robustness. These two objectives naturally lead to a tradeoff between \textit{exploitation} and \textit{exploration}, which we address through a hybrid model-based and data-driven control framework.

\subsubsection{State and Dynamics}
We define the system state at time $t$ as the LiDAR scanning angle $\theta_t \in \mathbb{R}$, representing the yaw orientation of the motorized LiDAR relative to the UAV body frame. The control input is the angular velocity $\omega_t \in \mathbb{R}$, which determines how fast the LiDAR rotates at each time step.
We adopt a first-order discrete-time dynamical model, consistent with real-time embedded control constraints and encoder-driven actuation:
\begin{equation}
\theta_{t+1} = \theta_t + \omega_t \cdot \Delta t,
\end{equation}
where $\Delta t$ is the control time interval. This model assumes constant angular velocity within each step and neglects higher-order rotational dynamics for computational simplicity. The full control trajectory $\omega_{t:t+T}$ is optimized over a fixed horizon $T$ within the MPC framework described in the following section.

\subsubsection{Observation and Perception Feedback}
At each control step, the AEOS system receives an observation $o_t$ composed of both geometric and estimation-related information derived from onboard sensors. Specifically, the observation includes:

\begin{itemize}
    \item A locally accumulated LiDAR point cloud $P_t$, constructed over a sliding time window using the current and recent scans, projected into the world frame using time-synchronized pose estimates and encoder-reported scan angles;
    \item An estimated pose uncertainty matrix $\Sigma_t$, obtained from the LiDAR-inertial odometry (LIO)(e.g., EKF), summarizing the current localization confidence.
\end{itemize}
Formally, the observation at time $t$ is defined as:
\begin{equation}
    o_t = \left[ \text{Enc}(P_t), \, \text{Vec}(\Sigma_t) \right], \label{eq:observation}
\end{equation}
where $\text{Enc}(P_t)$ denotes a geometric encoding of the local point cloud, and $\text{Vec}(\Sigma_t)$ is the vectorized form of the pose covariance matrix used for uncertainty reasoning.The local point cloud $P_t$ reflects the spatial structure and surface geometry of the scene, capturing occlusions, feature density, and visibility distribution, which influence the potential for future pose refinement. In parallel, the pose uncertainty $\Sigma_t$ serves as an exploitation signal, guiding the controller to direct scans toward regions that improve localization observability.

\subsubsection{Unified Panoramic Representation of Local Point Clouds for Downstream Modules} \label{sec:panoramic_representation}

To enable efficient integration of perception into both analytical and learned control objectives, we construct a unified representation of the local 3D scene using a panoramic depth map $\mathcal{D}_t \in \mathbb{R}^{H \times W}$. This map is obtained by projecting the local LiDAR point cloud $P_t$ onto a discretized spherical coordinate system centered at the UAV body frame as shown in Fig. \ref{fig:panoramic_depth}.

\begin{figure}
    \centering
    \includegraphics[width=\linewidth]{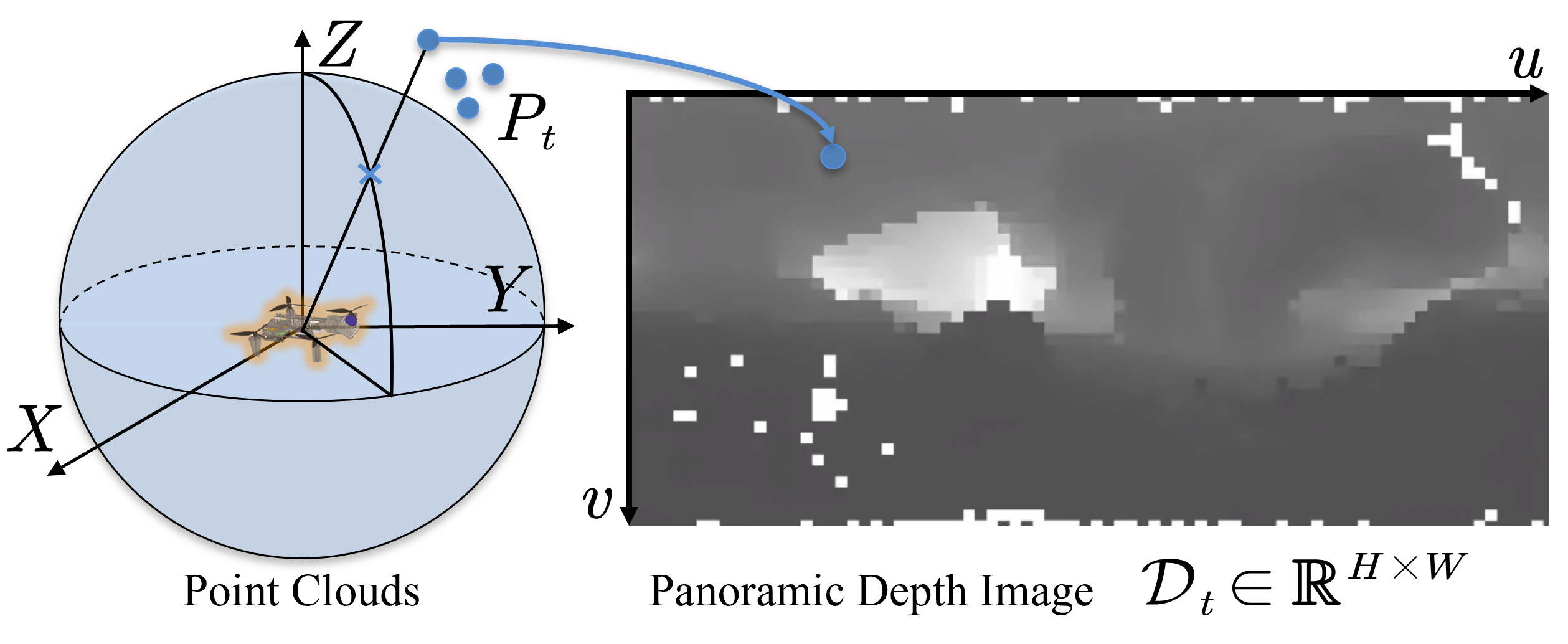}
    \caption{Unified Panoramic Representation of Local Point Clouds for Downstream Modules.}
    \label{fig:panoramic_depth}
\end{figure}

Each pixel in $\mathcal{D}_t$ encodes the shortest range observation within its azimuth-elevation bin, capturing occlusion patterns and surface structures relevant for planning. Formally, for each point $\mathbf{p} = [x, y, z]^\top \in P_t$, its image coordinates $(u, v)$ are given by:

\begin{equation}
\left[
\begin{array}{c}
u \\
v \\
\end{array}
\right] = \left[
\begin{array}{c}
2\pi + \frac{\arctan2(y,x)}{2\pi} \cdot W \\
\pi - \frac{\arcsin\left(z / \|\mathbf{p}\|\right)}{\pi} \cdot H \\
\end{array}
\right],
\label{eq:panomap_projection}
\end{equation}
where $W$ and $H$ are the width and height of the map. Importantly, $W$ and $H$ are configured separately for different downstream modules: higher resolution is used for uncertainty modeling (Section~\ref{sec:cost_map_exploitation}), while a coarser resolution suffices for neural cost learning (Section~\ref{sec:cost_map_exploration}) to reduce computational burden.
This design promotes representational consistency across modules and reduces redundant computation during real-time onboard execution.

The analytical uncertainty model (detailed in Section~\ref{sec:cost_map_exploitation}) predicts the impact of future viewpoints on positioning uncertainty, while a learned cost map (Section~\ref{sec:cost_map_exploration}) interprets $\mathcal{D}_t$ to favor exploration toward informative or under-observed areas. This perception-feedback loop enables AEOS to perform goal-driven and adaptive scanning in real time. The control objective is to select a scan trajectory that reduces future pose uncertainty while also encouraging exploration of informative or under-observed regions.
This problem structure naturally motivates a hybrid control solution: we combine an analytical model that explicitly predicts future uncertainty for exploitation, with a learned cost model that implicitly encodes exploration value from the point cloud. This hybrid design forms the basis of our decision-making architecture, detailed in the following section.

\subsection{Differential Hybrid RL-MPC Architecture}

\subsubsection{Decision Objective}
The control system selects a sequence of LiDAR scanning actions $\{\omega_k\}_{k=t}^{t+T}$ over a finite horizon $T$ to minimize a multi-objective cost that balances exploitation, exploration, and control smoothness. The optimization problem is formulated as:

\begin{equation}
\min_{\omega_{t:t+T}} \quad J = \sum_{k=t}^{t+T} \left( \mathcal{C}_{\text{unc}}(\theta_k) + \mathcal{C}_{\text{nn}}(\theta_k,\omega_k; \phi) + \rho \|\omega_k\|^2 \right),
\label{eq:mpc_objective}
\end{equation}
where $\mathcal{C}_{\text{unc}}(\theta_k)$ denotes an analytically computed uncertainty cost that quantifies expected localization inaccuracy, typically measured by the trace of the predicted pose covariance under scan direction $\theta_k$. The term $\mathcal{C}_{\text{nn}}(\theta_k,\omega_k; \phi)$ is a learned cost map generated by a neural network, conditioned on the current observation $o_t$, which captures spatial and semantic cues from the local point cloud and promotes exploratory behavior in under-observed regions. The final term, $\rho \|\omega_k\|^2$, penalizes high angular velocities to encourage smooth and energy-efficient motion.

The total cost $J$ is minimized by a model predictive control (MPC) solver that generates the optimal control sequence given current state estimates and physical constraints. Notably, the neural network does not produce actions directly, but instead modulates the cost landscape through $\mathcal{C}_{\text{nn}}$, allowing scene-aware adaptation while preserving the interpretability and task-awareness of model-based planning. The whole differential hybrid RL-MPC framework is illustrated in Fig. \ref{fig:abstract} (c).

\subsubsection{Model Predictive Control with Reinforcement Learning}

We model the adaptive viewpoint planning task as a Markov Decision Process (MDP), defined by the tuple $(\mathcal{S}, \mathcal{A}, \mathcal{P}, \mathcal{R}, \gamma)$, where $\mathcal{S}$ is the space of observations $o_t = [\text{Enc}(P_t)$, $\mathcal{A}$ is the space of LiDAR control actions $\omega_t$, $\mathcal{P}$ denotes the system dynamics $\theta_{t+1} = \theta_t + \omega_t \cdot \Delta t$, $\mathcal{R}$ is a task-dependent reward, and $\gamma$ is the discount factor. Unlike standard RL that directly maps observations to actions, AEOS adopts a structured formulation where the policy outputs a cost function $\mathcal{C}_{\text{nn}}(\theta,\omega; \phi)$ instead of an action. This learned cost modulates the MPC objective, guiding trajectory selection while preserving interpretability and ensuring constraint satisfaction.

At each timestep, the neural cost model parameterized by $\phi$ is queried based on $o_t$ and contributes to the composite cost in Eq.~\eqref{eq:mpc_objective}. The MPC solver then computes the optimal angular velocity $\omega_t$ by minimizing the total cost over the prediction horizon.
The neural cost function is trained via reinforcement learning to improve long-term task performance. Specifically, the training objective minimizes the expected critic value:

\begin{equation}
\mathcal{L}(\phi) = \mathbb{E}_{s \sim \mathcal{D}} \left[ Q(s, \omega^*(\phi)) \right],
\end{equation}
where $\omega^*(\phi)$ is the MPC-generated action given the current cost landscape. This actor-critic formulation enables gradient-based training of the policy through the differentiable MPC layer, allowing the learned cost to reflect both short-term constraints and long-term reward signals.

\subsubsection{Model Predictive Control Layer}

In AEOS, the MPC layer solves a constrained optimization problem over a finite horizon $T$ to generate the LiDAR control signal $\omega_t$. At each timestep, the objective is to minimize the total cost defined in Eq.~\eqref{eq:mpc_objective}:

\begin{equation}
\omega^*(\phi) = \arg\min_{\omega_{t:t+T}} \sum_{k=t}^{t+T} \left[ \mathcal{C}_{\text{unc}}(\theta_k) + \mathcal{C}_{\text{nn}}(\theta_k; \phi) + \rho \|\omega_k\|^2 \right],
\label{eq:mpc_solve}
\end{equation}
subject to the dynamics $\theta_{k+1} = \theta_k + \omega_k \cdot \Delta t$.
Under the assumption that the combined cost $\mathcal{C}(\theta_k; \phi) = \mathcal{C}_{\text{unc}}(\theta_k) + \mathcal{C}_{\text{nn}}(\theta_k,\omega_k; \phi)$ is differentiable with respect to $\theta_k$, the optimal control sequence $\boldsymbol{\omega}^*(\phi)$ satisfies the first-order optimality condition:

\begin{equation}
\boldsymbol{\omega}^*(\phi) = -\frac{1}{2\rho} A^\top \nabla_{\boldsymbol{\theta}} \mathcal{C}(\boldsymbol{\theta}; \phi),
\label{eq:mpc_analytic_solution}
\end{equation}
where $\boldsymbol{\omega} = [\omega_t, \omega_{t+1}, \dots, \omega_{t+T}]^\top \in \mathbb{R}^{T+1}$ is the control input sequence over the planning horizon;
$\boldsymbol{\theta} = [\theta_{t+1}, \theta_{t+2}, \dots, \theta_{t+T+1}]^\top \in \mathbb{R}^{T+1}$ is the predicted future LiDAR scan angles;
$A \in \mathbb{R}^{(T+1) \times (T+1)}$ is a strictly lower-triangular integration matrix that maps control inputs to angle increments:

\begin{equation}
A = \Delta t \cdot
\begin{bmatrix}
1 & 0 & 0 & \cdots & 0 \\
1 & 1 & 0 & \cdots & 0 \\
1 & 1 & 1 & \cdots & 0 \\
\vdots & \vdots & \vdots & \ddots & 0 \\
1 & 1 & 1 & \cdots & 1
\end{bmatrix}, \qquad
\boldsymbol{\theta} = \theta_t \cdot \mathbf{1} + A \boldsymbol{\omega}.
\end{equation}

The matrix $A$ accumulates the angular velocities $\boldsymbol{\omega}$ to compute future scan angles under the first-order model $\theta_{k+1} = \theta_k + \omega_k \Delta t$. Substituting $\boldsymbol{\theta}$ into Eq.~\eqref{eq:mpc_analytic_solution} yields a fixed-point system, which can be solved iteratively or approximated using unrolled gradient descent.
Once solved, the first control value $\omega_t^*$ is applied to the system and used for backpropagation during training.

After applying $\omega_t = \omega^*_t(\phi)$ to the system, the environment returns a task-specific reward $r_t$ or critic estimate $Q(o_t, \omega_t)$. The policy parameters $\phi$ are optimized to minimize the expected critic value:
$\mathcal{L}(\phi) = Q(o_t, \omega^*(\phi))$.

Since the optimal control $\omega^*(\phi)$ is implicitly defined as the solution to the constrained optimization problem in Eq.~\eqref{eq:mpc_solve}, its dependency on the neural parameters $\phi$ is indirect, through the learned cost function $\mathcal{C}_{\text{nn}}(\cdot; \phi)$. To compute the policy gradient $\nabla_\phi \mathcal{L}(\phi)$, we apply the implicit function theorem to differentiate through the solution of the MPC layer.

Let $J(\boldsymbol{\omega}; \phi)$ denote the MPC objective over horizon $T$, and assume it is twice continuously differentiable in both $\boldsymbol{\omega}$ and $\phi$. At optimality, the control sequence $\boldsymbol{\omega}^*(\phi)$ satisfies the first-order stationarity condition as follows.
\begin{equation}
\nabla_{\boldsymbol{\omega}} J(\boldsymbol{\omega}^*(\phi); \phi) = 0.
\label{eq:kkt_condition}
\end{equation}
This defines an implicit relationship between $\boldsymbol{\omega}^*$ and $\phi$, and enables gradient propagation through the MPC optimization layer via the following result:

\begin{thm}[Policy Gradient via Differentiable MPC]
\label{thm:mpc_grad}
\;\\
Let $\boldsymbol{\omega}^*(\phi)$ be the unique minimizer of $J(\boldsymbol{\omega}; \phi)$, and let the policy loss be $\mathcal{L}(\phi) = Q(o_t, \omega_t^*(\phi))$ for a differentiable critic $Q$. Then, under standard regularity conditions, the policy gradient is:
\begin{equation}
\nabla_\phi \mathcal{L}(\phi) = \nabla_\omega Q(o_t, \omega_t^*(\phi)) \cdot \mathbf{e}_1^\top \cdot \left( -\left( \nabla^2_{\boldsymbol{\omega} \boldsymbol{\omega}} J \right)^{-1} \cdot \nabla^2_{\phi \boldsymbol{\omega}} J \right),
\end{equation}
where $\mathbf{e}_1$ selects the first entry of the optimal control sequence.
\end{thm}

\begin{proof}
By the implicit function theorem applied to Eq.~\eqref{eq:kkt_condition}, the total derivative of $\boldsymbol{\omega}^*(\phi)$ is given by:
\begin{equation}
\frac{d \boldsymbol{\omega}^*}{d \phi} = - \left( \nabla^2_{\boldsymbol{\omega} \boldsymbol{\omega}} J \right)^{-1} \cdot \nabla^2_{\phi \boldsymbol{\omega}} J.
\end{equation}
Applying the chain rule to $\mathcal{L}(\phi) = Q(o_t, \omega_t^*(\phi))$ and noting that $\omega_t^*(\phi)$ is the first element of $\boldsymbol{\omega}^*(\phi)$, we obtain:
\begin{align}
\nabla_\phi \mathcal{L}(\phi) 
&= \nabla_{\omega} Q(o_t, \omega_t^*(\phi)) \cdot \frac{d \omega_t^*}{d \phi} \nonumber \\
&= \nabla_{\omega} Q(o_t, \omega_t^*(\phi)) \cdot \mathbf{e}_1^\top \cdot \frac{d \boldsymbol{\omega}^*}{d \phi}.
\end{align}
Substituting the expression for the Jacobian completes the proof.
\end{proof}

To realize the hybrid RL-MPC framework described above, we now detail its two complementary components: an analytical uncertainty model for exploitation-aware viewpoint selection (Section \ref{sec:cost_map_exploitation}) and a reinforcement learning-based cost map for promoting exploration (Section \ref{sec:cost_map_exploration}).

\subsection{Analytical Pose Uncertainty Modeling for Exploitation}\label{sec:cost_map_exploitation}

\begin{figure}
    \centering
    \includegraphics[width=\linewidth]{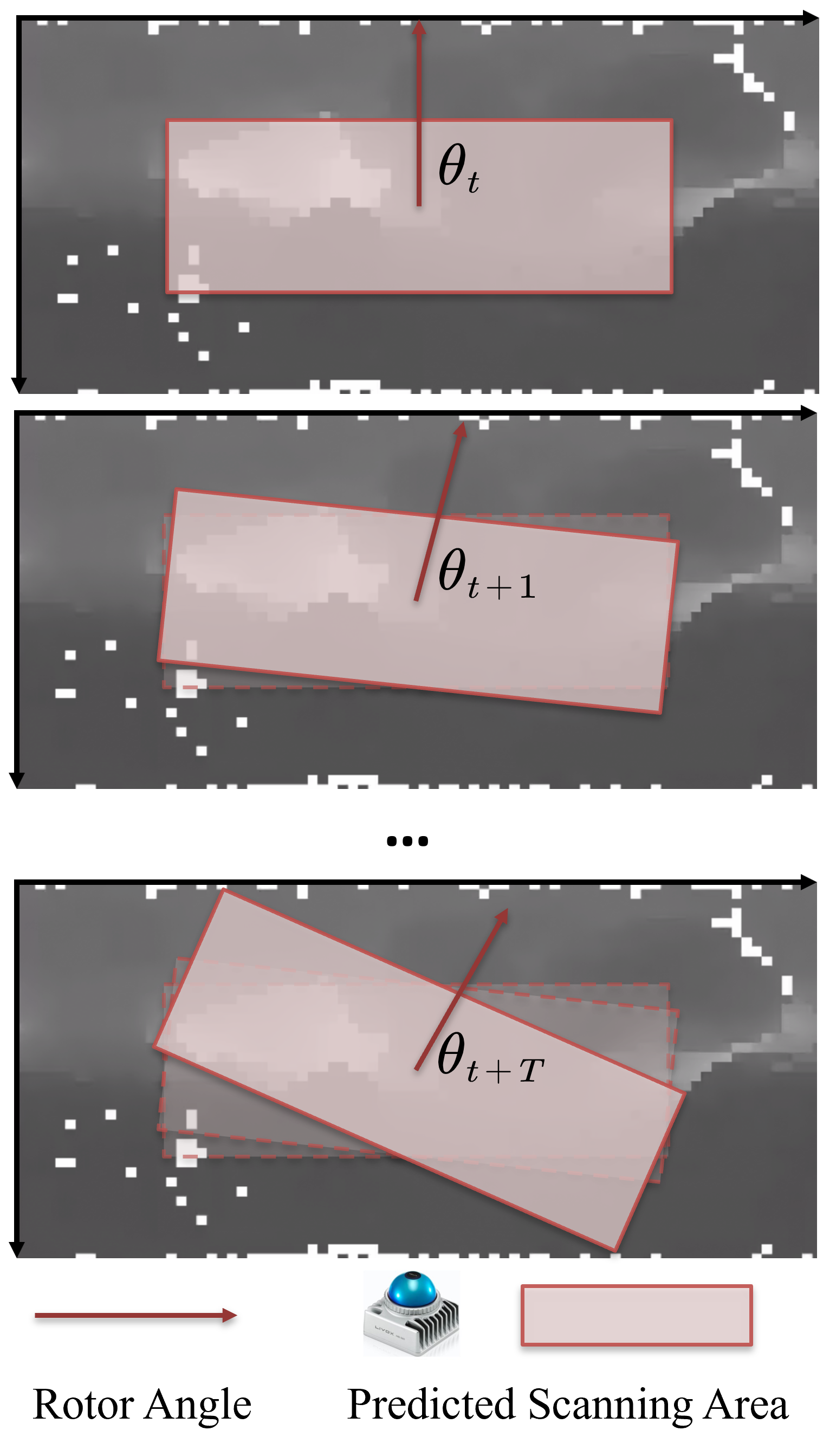}
    \caption{Scanning prediction in a control horizon using the unified panoramic depth map.}
    \label{fig:scanning_prediction}
\end{figure}

To guide AEOS tracking of existing scanned geometry-rich regions that improve localization accuracy, we incorporate an analytical uncertainty cost into the MPC objective. This cost is derived from a predictive model of future pose uncertainty, using surrogate observability optimization adapted from UA-MPC~\citep{li2025ua}.

\subsubsection{Uncertainty Prediction via Predictive LiDAR Observability}

At each planning step $k$, we estimate the expected impact of the LiDAR scan direction $\theta_k$ on future pose uncertainty. Let $\Sigma_k \in \mathbb{R}^{6 \times 6}$ denote the predicted covariance of the LIO state at time $t_k$. To approximate $\Sigma_k$, we raycast a virtual LiDAR scan in direction $\theta_k$ using the unified panoramic depth representation of the local map $P_t$ introduced in Section~\ref{sec:panoramic_representation}. This representation allows us to efficiently extract synthetic point clouds $\mathcal{R}(\theta_k)$ by cropping a directional window aligned with $\theta_k$ according to the LiDAR's field of view (FoV) as shown in Fig. \ref{fig:scanning_prediction}. The rendered points are then used to evaluate the observability of the scan direction.

For each rendered point $\mathbf{p}_j$ in the LiDAR frame, we compute its corresponding 3D location $\mathbf{P}_j$ in the world frame and the associated surface normal $\mathbf{n}_j$. Following the point-to-plane formulation used in LiDAR odometry, we define the residual of the $j$-th point as:
\begin{equation}
\epsilon_j = \mathbf{n}_j^\top \left( \hat{\mathbf{R}}^W_B(t_k) \mathbf{p}_j + \hat{\mathbf{r}}^W_B (t_k) - \mathbf{P}_j \right),
\label{eq:aeos_residual}
\end{equation}
where $\hat{\mathbf{R}}^W_B(t_k)$ and $\hat{\mathbf{r}}^W_B(t_k)$ denote the estimated rotation and translation of the UAV body frame with respect to the world frame at time $t_k$. The Jacobian of the residual with respect to the 6-DoF pose is approximated as:
\begin{equation}
\begin{aligned}
\mathbf{J}_k &= [(\partial\epsilon_k /\partial \hat{\mathbf{R}}^W_{B}(t_k))^\top,(\partial\epsilon_k /\partial \hat{\mathbf{r}}^W_{B}(t_k))^\top]^\top \\
&= [[(\hat{\mathbf{R}}^W_{B}(t_k)\mathbf{p}_k]_{\times}\mathbf{n}_k)^\top,\mathbf{n}^\top_k]^\top. \label{eq:aeos_jacobian}
\end{aligned}
\end{equation}
where $[\cdot]_{\times}$ denotes the standard skew-symmetric operator mapping $\mathbb{R}^3 \to \mathbb{R}^{3 \times 3}$.

According to A-optimal design theory~\cite{pukelsheim2006optimal}, the expected localization uncertainty under a candidate LiDAR scan angle $\theta_k$ is quantified by the trace of the inverse information matrix:

\begin{subequations}
\begin{align}
U(\theta_k) &= \mathrm{tr} \left( \boldsymbol{\Lambda}_k^{-1} \right), \label{eq:aoptimal_trace} \\
\boldsymbol{\Lambda}_k &= \sum_{j=1}^{M} \mathbf{J}_j \mathbf{J}_j^\top,
\end{align}
\end{subequations}
where $U(\theta_k)$ is the expected localization uncertainty under a candidate LiDAR scan angle $\theta_k$. $\boldsymbol{\Lambda}_k$ is the surrogate Fisher information matrix constructed from the rendered LiDAR returns $\{\mathbf{p}_j\}_{j=1}^{M}$ at scan angle $\theta_k$, and $\mathbf{J}_j$ is the Jacobian of the residual at point $\mathbf{p}_j$ with respect to the UAV pose as defined in Eq.~\eqref{eq:aeos_jacobian}. $U(\theta_k)$ can not be directly used in the MPC objective due to two major reasons: the evaluation of $U(\theta_k)$ involves LiDAR rendering and Jacobian accumulation, which is computationally expensive; more critically, $U(\theta_k)$ is a non-differentiable function with respect to $\theta_k$, violating the gradient consistency required for backpropagation through the differentiable MPC layer (Theorem~\ref{thm:mpc_grad}).
\subsubsection{Uncertainty Cost for Exploitation in MPC}
To address the above problem, we approximate $U(\theta_k)$ using a differentiable surrogate cost $\mathcal{C}_{\text{unc}}(\theta_k)$ constructed via piecewise linear interpolation over pre-sampled angles. This surrogate enables the uncertainty term to be incorporated into the MPC optimization as a smooth and tractable cost function.
At runtime, we sample $U(\theta)$ at discrete angles within the control domain using a fixed interval $\Delta \theta$. This yields $N = 2\pi / \Delta \theta$ precomputed uncertainty values:
\begin{equation}
\left\{ U_s = U(\theta_t + s \cdot \Delta \theta) \mid s = 0, 1, \dots, N-1 \right\}.
\label{eq:uncertainty_samples_updated}
\end{equation}
For an arbitrary scan angle $\theta_k \in [0, 2\pi)$ within the MPC planning horizon, we approximate its uncertainty as:
\begin{equation}
\mathcal{C}_{\text{unc}}(\theta_k) =
(1 - \delta) \cdot U_{\lfloor i \rfloor} + \delta \cdot U_{\lfloor i \rfloor + 1}, \label{eq:uncertainty_interp_updated}
\end{equation}
where $i = \theta_k / \Delta \theta$ and $\delta = i - \lfloor i \rfloor$ is the interpolation weight. The interpolated function $\mathcal{C}_{\text{unc}}$ is piecewise linear and hence differentiable almost everywhere. Its gradient with respect to $\theta_k$ is:
\begin{equation}
\frac{\partial \mathcal{C}_{\text{unc}}}{\partial \theta_k} =
\frac{U_{\lfloor i \rfloor + 1} - U_{\lfloor i \rfloor}}{\Delta \theta}. \label{eq:uncertainty_grad_updated}
\end{equation}

During each MPC step, we cache the sample set $\{U_s\}$ computed from the current local LiDAR map and apply Eq.~\eqref{eq:uncertainty_interp_updated} and \eqref{eq:uncertainty_grad_updated} to evaluate $\mathcal{C}_{\text{unc}}(\theta_k)$ and its gradient. This approximation allows the uncertainty-aware exploitation signal to be smoothly integrated into the optimization objective in Eq.~\eqref{eq:mpc_objective}, while enabling efficient end-to-end policy learning via differentiable MPC.

\begin{rmk}[Justification of Linear Surrogate]
To enable differentiable and efficient integration into the MPC layer, we adopt a piecewise linear surrogate $\mathcal{C}_{\text{unc}}(\theta)$ constructed from sparse evaluations of $U(\theta)$.
This approximation is well-justified by the physical properties of rotating LiDAR systems. As the LiDAR rotates smoothly over time, consecutive scan directions $\theta_k$ and $\theta_{k+1}$ yield highly overlapping point clouds. Since observability in pose estimation is governed by the geometric configuration and spatial distribution of these returns, adjacent viewpoints naturally produce similar uncertainty estimates. Consequently, the true uncertainty function $U(\theta)$ varies smoothly and slowly over $\theta$, making it well-suited for approximation by linear interpolation between sampled points.
\end{rmk}

\begin{rmk}[Why Analytical Modeling for Exploitation]
A natural alternative to the proposed uncertainty model is to use an end-to-end neural network to directly infer the exploitation cost from raw or downsampled point clouds. However, this approach is impractical for real-time UAV deployment for two key reasons.
First, the onboard LiDAR generates dense point clouds at high frequency, leading to a substantial computational burden if directly fed into deep perception networks without aggressive sampling. This overhead is incompatible with the real-time constraints and limited compute resources of aerial platforms.
Second, aggressive downsampling of the point cloud results in significant loss of geometric detail, which is critical for accurately estimating localization observability. In contrast, the proposed analytical pose uncertainty model retains the full resolution of geometric information and directly quantifies localization-relevant metrics via Jacobian-based observability analysis.
Therefore, by incorporating a physically grounded, differentiable approximation of pose uncertainty, we effectively capture exploitation-relevant structure while maintaining low computational cost and strong interpretability.
\end{rmk}

\subsection{Reinforcement Learning-Based Cost Map for Exploration} \label{sec:cost_map_exploration}
To complement the analytical uncertainty model used for exploitation, we introduce a learned cost map that promotes exploratory scanning behavior. This cost map is derived from local point cloud geometry and encodes task-agnostic priors about visibility, occlusion, and structural complexity. The cost function $\mathcal{C}_{\text{nn}}(\theta_k, \omega_k; \phi)$ is implemented as a neural network policy that takes the panoramic depth map and system states as input. 

\subsubsection{Neural Encoding and Cost Map Generation}
The neural exploration module aims to assign a scan utility cost $\mathcal{C}_{\text{nn}}(\theta_k, \omega_k; \phi)$ based on the local scene structure and system states, encouraging the LiDAR to explore informative, less-observed regions as shown in Fig.\ref{fig:network}. Instead of operating directly on dense point clouds, we use the panoramic depth map $\mathcal{D}_k \in \mathbb{R}^{H \times W}$ introduced in Section~\ref{sec:panoramic_representation} as a compact and efficient input representation (downspample to $H=80, W=40$). More specifically, we feed the system states from LIO (velocity in body frame $\mathbb{R}^{3}$, diagonal elements of positional covariance $\mathbb{R}^{3}$), current rotor angle ($\mathbb{R}^{1}$) and low-resolution panoramic depth map ($\mathbb{R}^{80\times 40}$) to the neural network, and construct the learned cost map $\mathcal{C}_{\text{nn}}(\theta_k, \omega_k; \phi)$.

\begin{figure}[]
    \centering
    \includegraphics[width=0.9\linewidth]{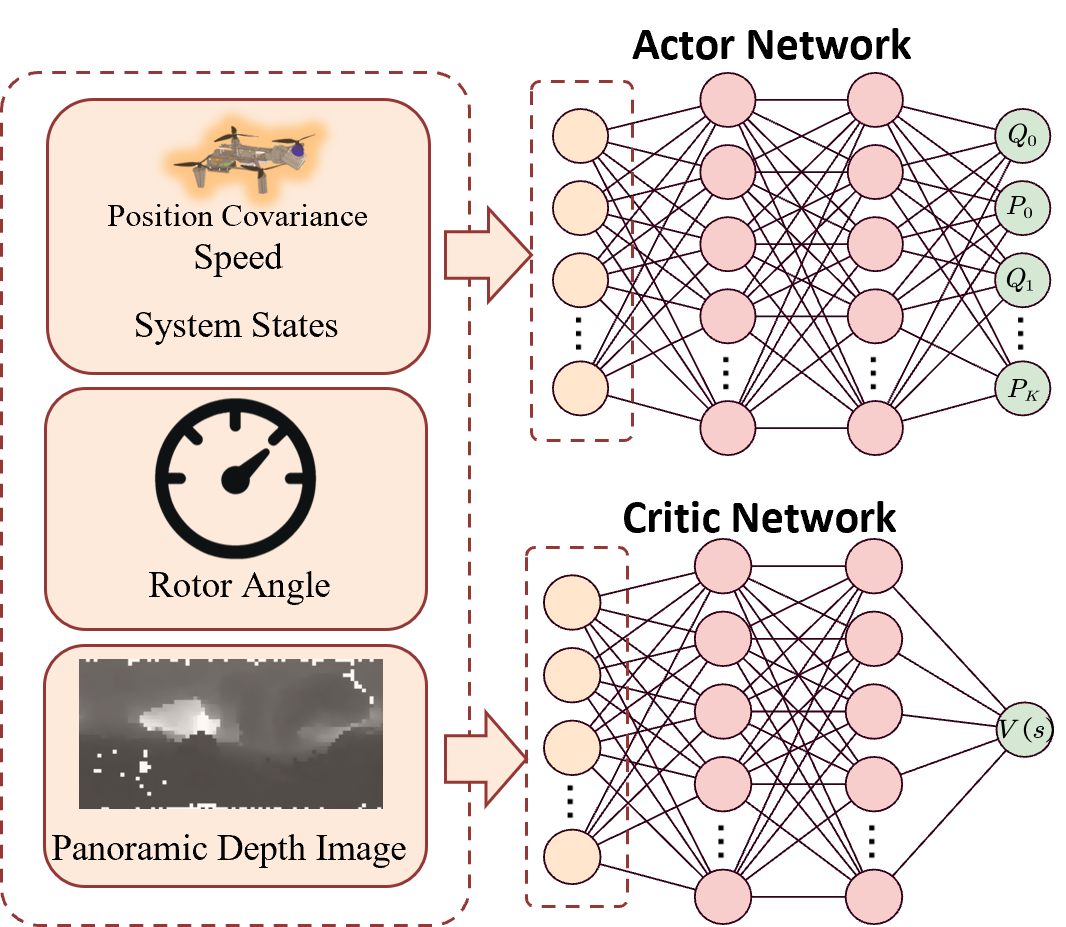}
    \caption{Light-weight neural encoding for the implicit cost map.}
    \label{fig:network}
\end{figure}

We model the learned cost map $\mathcal{C}_{\text{nn}}(\theta_k, \omega_k; \phi)$ as a structured quadratic function over the control variables $\theta_k$ and $\omega_k$. Specifically, we define:
\begin{equation}
\mathcal{C}_{\text{nn}}(\theta_k, \omega_k; \phi) =
\frac{1}{2}
\begin{bmatrix}
\theta_k \\
\omega_k
\end{bmatrix}^\top
\mathbf{Q}_k(\phi)
\begin{bmatrix}
\theta_k \\
\omega_k
\end{bmatrix}
+
\mathbf{q}_k^\top(\phi)
\begin{bmatrix}
\theta_k \\
\omega_k
\end{bmatrix},
\label{eq:cnn_cost_parametric}
\end{equation}
where, $\mathbf{Q}_k(\phi) \in \mathbb{R}^{2 \times 2}$ is a symmetric matrix predicted by the neural network. $\mathbf{q}_k(\phi) \in \mathbb{R}^{2}$ is a linear coefficient vector. Both of $\mathbf{Q}_k(\phi)$ and  $\mathbf{q}_k(\phi)$ are conditioned on the unified depth map $\mathcal{D}_k$. The network $f_\phi$ is shared across all horizon steps and produces both components:
\begin{equation}
\left( \mathbf{Q}_k, \mathbf{q}_k \right) = f_\phi\left( \mathcal{D}_k \right).
\end{equation}

\begin{rmk}[Convexity of Neural Cost Map]
To ensure both computational efficiency and convexity in the neural exploration cost $\mathcal{C}_{\text{nn}}(\theta_k, \omega_k; \phi)$, we parameterize the cost as a quadratic function over the control variables $\begin{bmatrix} \theta_k & \omega_k \end{bmatrix}^\top$ using learnable matrices $(\mathbf{Q}_k, \mathbf{q}_k)$ as defined in Eq.~\eqref{eq:cnn_cost_parametric}.
We enforce convexity by restricting $\mathbf{Q}_k(\phi)$ to be positive semi-definite. To reduce the dimensionality of the learnable parameter space and ensure efficient optimization, we constrain $\mathbf{Q}_k$ to be diagonal:
\begin{equation}
\mathbf{Q}_k(\phi) = \mathrm{diag}(q_{\theta,k},\ q_{\omega,k}),
\end{equation}
where $q_{\theta,k}$ and $q_{\omega,k}$ are scalar values output by the neural network. To ensure non-negativity and numerical stability, we apply a scaled sigmoid activation to bound the outputs:
\begin{equation}
q_{\cdot,k} = q_{\min} + (q_{\max} - q_{\min}) \cdot \mathrm{sigmoid}(\hat{q}_{\cdot,k}),
\end{equation}
with $q_{\min} = 0.1$ and $q_{\max} = 10^5$. $\hat{q}_{\cdot,k}$ is the output of the multi-layer perception network.

Similarly, the linear coefficients $\mathbf{q}_k(\phi) \in \mathbb{R}^2$ are also bounded via sigmoid scaling to avoid overly aggressive exploration. The final neural cost map consists of two hidden layers of width 256 with ReLU activations, and a sigmoid output layer to produce bounded $(\mathbf{Q}_k, \mathbf{q}_k)$ for each timestep.
This design guarantees that the resulting cost $\mathcal{C}_{\text{nn}}$ is convex with respect to control inputs, while allowing the network to learn task-dependent exploratory preferences under gradient-based training.
\end{rmk}

This formulation allows the neural network to capture high-level environmental priors such as occlusion, feature richness, and spatial coverage, while preserving compatibility with MPC's optimization structure. The quadratic form enables efficient gradient computation and stable integration into the differentiable control pipeline.

\subsubsection{Reward Functions}
The neural cost map policy $f_\phi$ is trained via reinforcement learning, where the reward function jointly incentivizes effective exploration and robust localization. At each timestep $t$, the total reward is defined as:
\begin{equation}
r_t = \lambda_{\text{exp}} \cdot r_{\text{exp}}(t) + \lambda_{\text{lio}} \cdot r_{\text{lio}}(t),
\end{equation}
where $r_{\text{exp}}(t)$ encourages the discovery of novel regions, and $r_{\text{lio}}(t)$ promotes trajectory consistency as estimated by the LIO backend. The weighting terms $\lambda_{\text{exp}}$ and $\lambda_{\text{lio}}$ control the trade-off between exploration and localization.

To measure exploration gain, we define $r_{\text{exp}}(t)$ as the ratio of newly observed voxels to the total visible voxels in the current LiDAR scan:
\begin{equation}
r_{\text{exp}}(t) = \frac{|\mathcal{V}_t^{\text{new}}|}{|\mathcal{V}_t|},
\end{equation}
where $\mathcal{V}_t$ denotes all voxels observed at time $t$, and $\mathcal{V}_t^{\text{new}}$ represents the subset not previously seen in the accumulated map. This term encourages the policy to steer the sensor toward previously occluded or unobserved areas.

Localization fidelity is captured through the reward term:
\begin{equation}
r_{\text{lio}}(t) = 1/\mathrm{RTE}(t, \tau),
\end{equation}
where $\mathrm{RTE}(t, \tau)$ denotes the average relative translational error over a sliding window of duration $\tau$, computed between the ground truth and the estimated trajectory from the LIO system. This term provides a dense, continuous signal that directly reflects the impact of viewpoint selection on downstream pose estimation accuracy.

\section{Point Cloud-Based Scanning Control Simulation and Training}\label{sec:simulation}

To train and evaluate the AEOS control policy in a reproducible and safe manner, we develop a lightweight point cloud-based simulation framework built upon ROS. As shown in Fig.~\ref{fig:simulation_env}, the simulation emulates LiDAR scanning, vehicle motion, and real-time control signals, using real-world datasets as ground truth.

\begin{figure}[]
    \centering
    \includegraphics[width=0.9\linewidth]{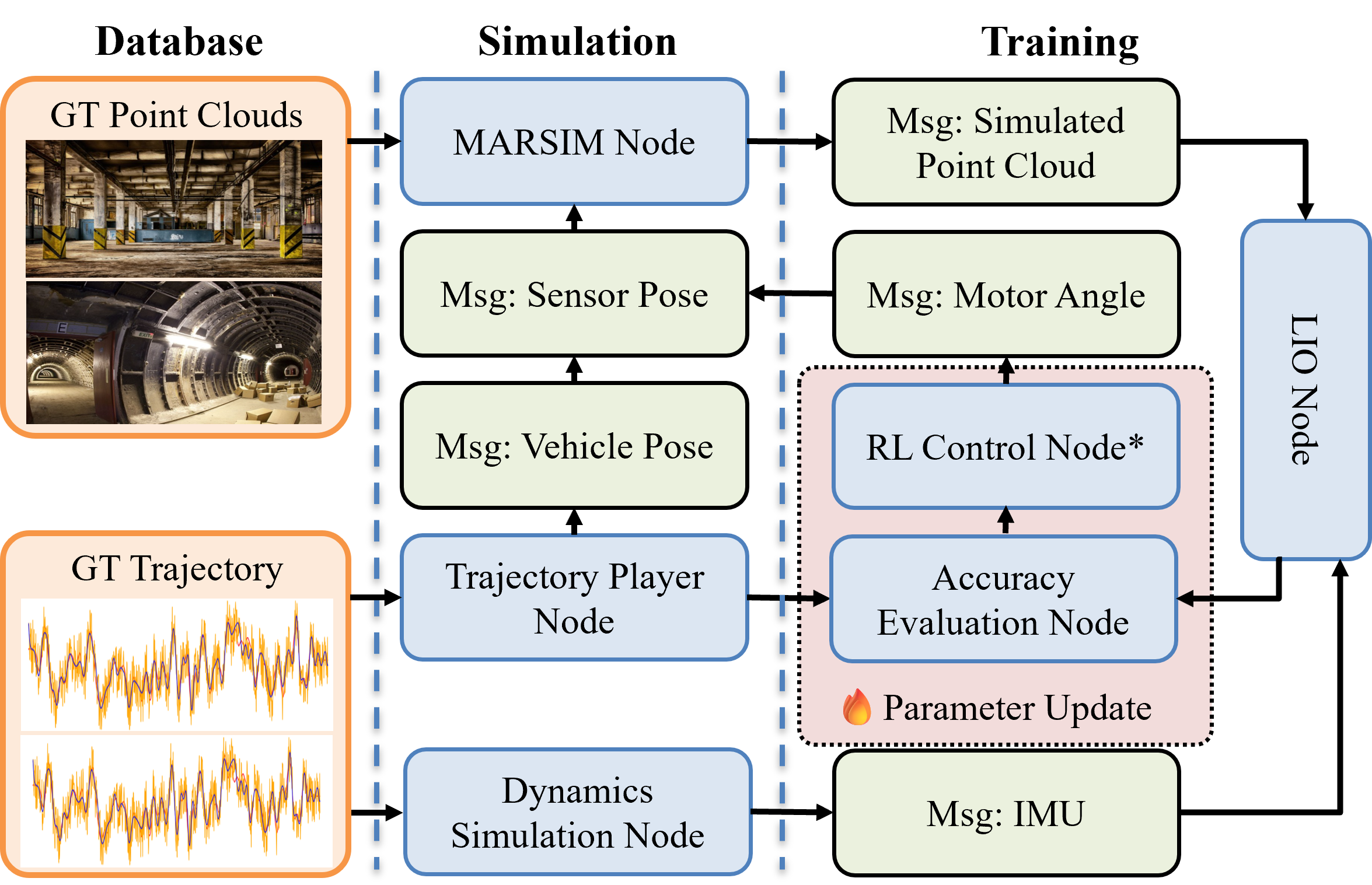}
    \caption{System architecture of the point cloud-based simulation environment. Ground-truth data are used to simulate LiDAR observations and UAV dynamics, enabling reinforcement learning with real-time feedback from the LIO backend.}
    \label{fig:simulation_env}
\end{figure}

\subsection{Simulation Architecture}

The simulation begins with a real-world database consisting of globally registered LiDAR point clouds and time-synchronized ground-truth trajectories. A \textit{trajectory player node} streams poses into the \textit{dynamics simulation node}, which introduces control delay and UAV motion constraints. This ensures the UAV follows realistic trajectories when executing control commands.

Given the simulated UAV pose and motorized angle $\theta_t$, the \textit{MARSIM node} raycasts virtual LiDAR scans from the global point cloud, producing a synthetic scan consistent with LiDAR field-of-view and occlusion geometry. The resulting point cloud is published to both the LIO backend (e.g., Fast-LIO2~\cite{xu2022fast}) and the control module.

The \textit{RL control node} receives observations and outputs the angular velocity command $\omega_t$, which updates the scan direction. Meanwhile, an \textit{accuracy evaluation node} computes rewards for both exploration and localization, using either pose error (via ground-truth) or estimation uncertainty (from the LIO output). This reward is sent back for policy learning via parameter updates.

This modular architecture allows fast prototyping of control strategies in diverse environments, while faithfully preserving the LiDAR observability structure and feedback loop seen in real-world systems.

\subsection{Simulation Dataset}

\begin{figure}[]
    \centering
    \includegraphics[width=0.9\linewidth]{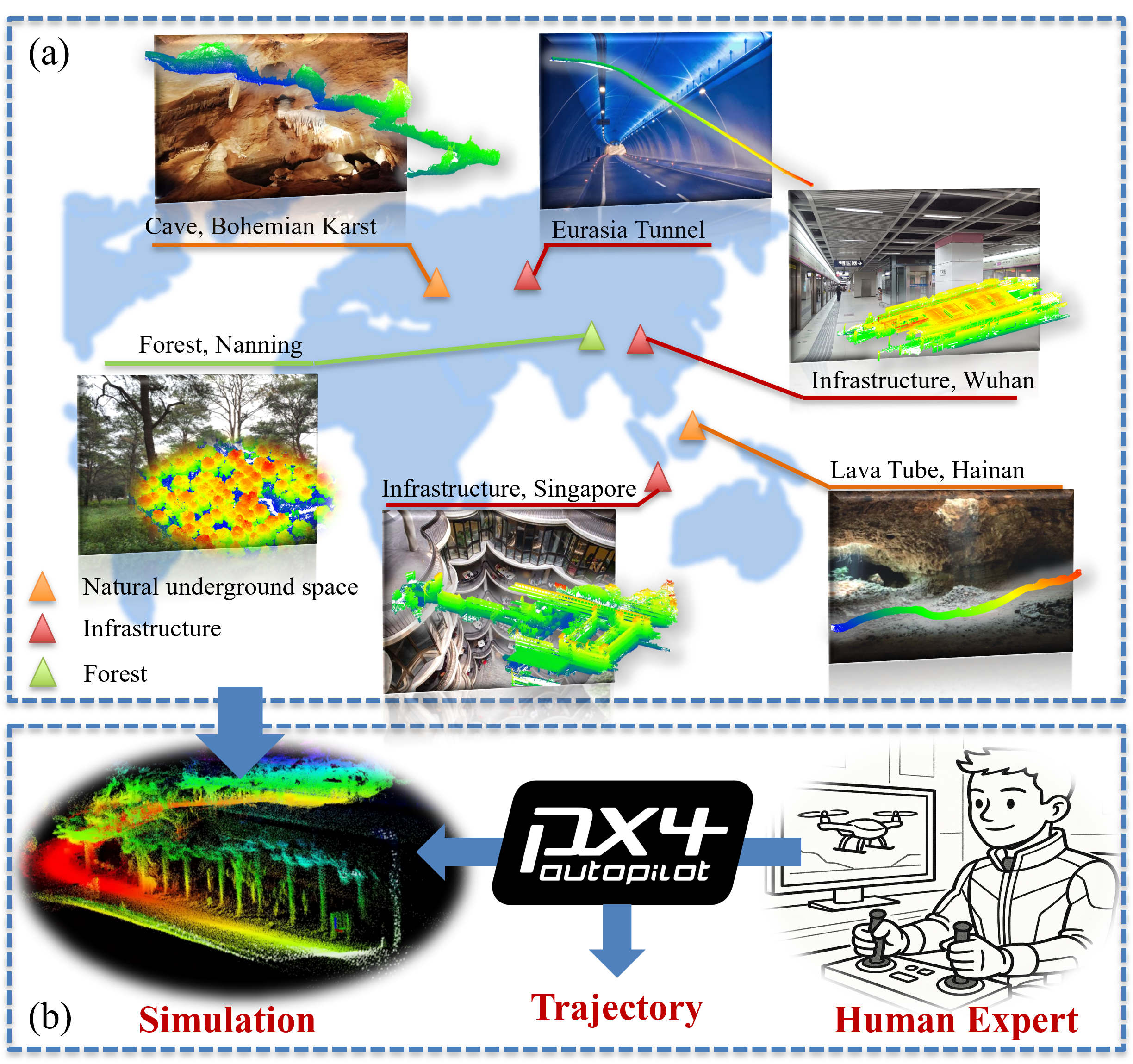}
    \caption{Overview of the simulation dataset used for active scanning. (a) High-fidelity point cloud maps of diverse real-world environments, including caves, tunnels, subways, forests, and lava tubes across the globe. (b) The simulation framework integrates the PX4 autopilot, enabling human experts to control the UAV and record realistic trajectories in complex environments.}
    \label{fig:datasets}
\end{figure}

To enable simulation-driven training, we construct a synthetic dataset based on high-resolution 3D scans of large-scale, complex environments from various locations worldwide, as shown in Fig.~\ref{fig:datasets}. Each dataset instance includes a dense global point cloud map $\mathcal{M}_{\text{GT}}$ and a ground-truth trajectory $\{\mathbf{T}_t^W\}$, collected via human-expert control in a PX4-based simulation environment.

The selected scenes represent typical domains where UAV-based exploration is essential, such as subterranean infrastructure, dense vegetation, and enclosed industrial corridors. These environments are not only operationally relevant for autonomous missions (e.g., search and rescue, infrastructure inspection, and remote monitoring), but also present substantial challenges to SLAM systems, due to severe occlusion, limited structural features, and narrow spaces.

The dataset spans geometrically diverse and feature-rich scenes designed to evaluate viewpoint planning under a wide range of sensing difficulties. Detailed statistics of all scenes are listed in Table~\ref{tab:dataset_stats}.

\begin{table*}[t]
    \centering
    \caption{Description of the simulation dataset.}
    \label{tab:dataset_stats}
    \begin{adjustbox}{width=\linewidth}
    \begin{tabular}{lcccccc}
        \toprule
        \textbf{Sequence} & \textbf{Scene} & \textbf{Location} & \textbf{Environment Type}& \textbf{Average Velocity (m/s)} & \textbf{Trajectory Time (s)} & \textbf{Trajectory Length (m)} \\
        \midrule
        Simu-Seq01 & Lava Tube     & Hainan, China        & Natural underground space & 2.30 &119 & 275\\
        Simu-Seq02 &Cave          & Bohemian Karst, Czech       & Natural underground space  &  2.52 & 219 & 552   \\
        Simu-Seq03 &Eurasia Tunnel & Istanbul, Turkey    & Infrastructure        & 7.05 & 313 & 2207 \\
        Simu-Seq04 &Wuhan Tunnel        & Wuhan, China         & Infrastructure  &2.40 & 379 &158  \\
        Simu-Seq05 &Wuhan Subway        & Wuhan, China         & Infrastructure  & 2.20& 153 & 335\\
        Simu-Seq06 &Building & Singapore, Singapore  & Infrastructure  & 2.20 & 304 & 661 \\
        Simu-Seq07 &Spine & Singapore, Singapore  & Infrastructure  & 1.02  & 942 & 962 \\
        Simu-Seq08 &Forest        & Nanning, China       & Forest         & 2.67 & 414 & 1107 \\
        \bottomrule
    \end{tabular}
    \end{adjustbox}
\end{table*}

\subsection{Training Configuration}

We train our hybrid RL-MPC policy using the Proximal Policy Optimization (PPO) algorithm within the custom ROS-based simulation environment. We select only 40\% length of the trajectory for each dataset as training data. In each episode, 60 second-length is extracted from the training. Training is conducted with a total of 500,000 environment steps in each scene. The optimizer is Adam with a linearly scheduled learning rate. We apply gradient clipping with $\|\nabla\| \leq 0.5$ and enforce a KL divergence threshold of 0.01 to stabilize updates. All experiments are performed on a CPU backend.

\begin{figure*}[h!]
    \centering
    \includegraphics[width=0.9\linewidth]{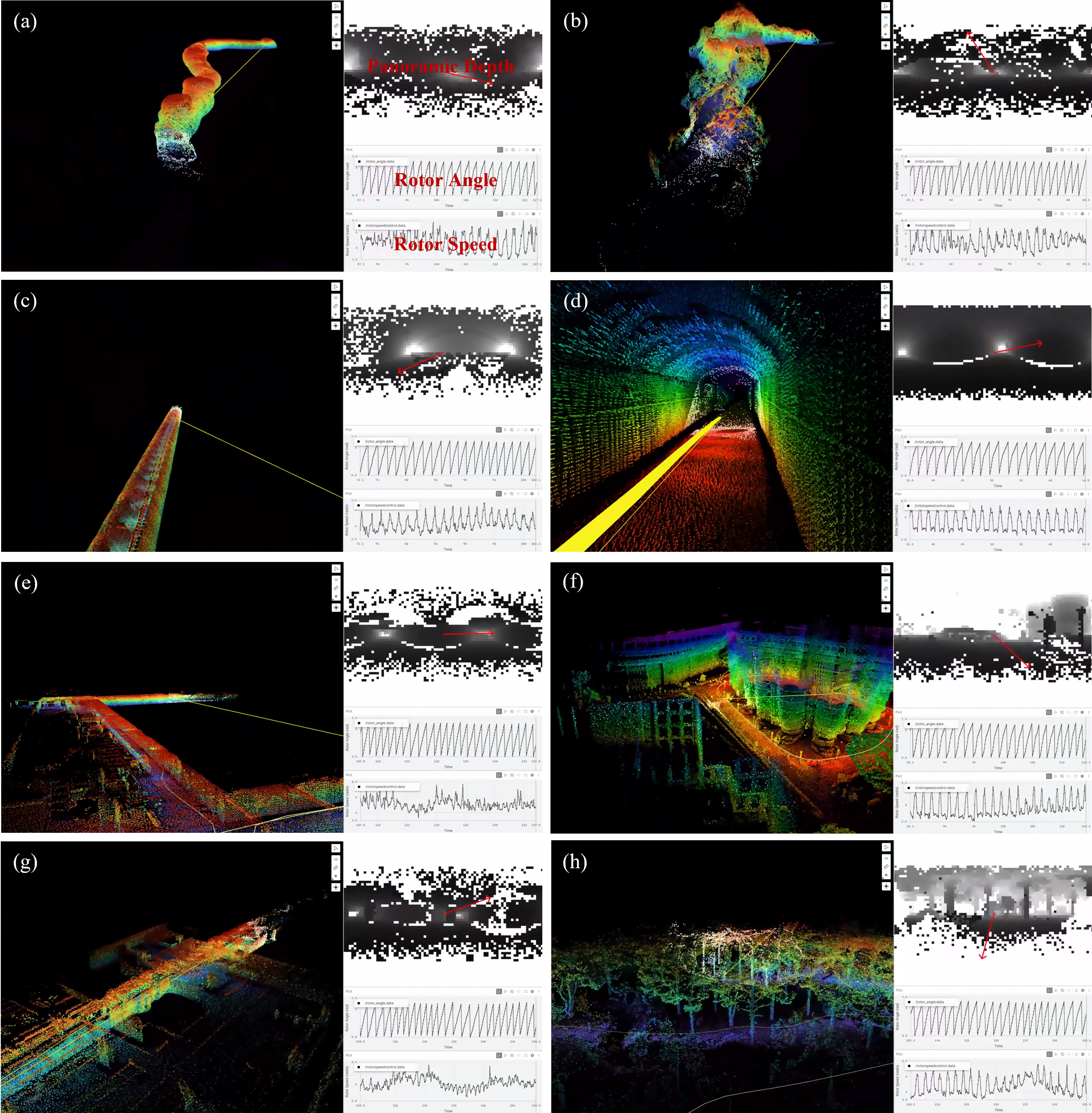}
    \caption{Simulation experiments. (a)-(g), the screenshots of the simulation environments for the sequences Simu-Seq01 to Simu-Seq08. Each figure includes the point clouds, panoramic depth, rotor angle serial, and rotor speed serial.}
    \label{fig:simu_experiments}
\end{figure*}

\begin{figure*}[]
    \centering
    \includegraphics[width=0.8\linewidth]{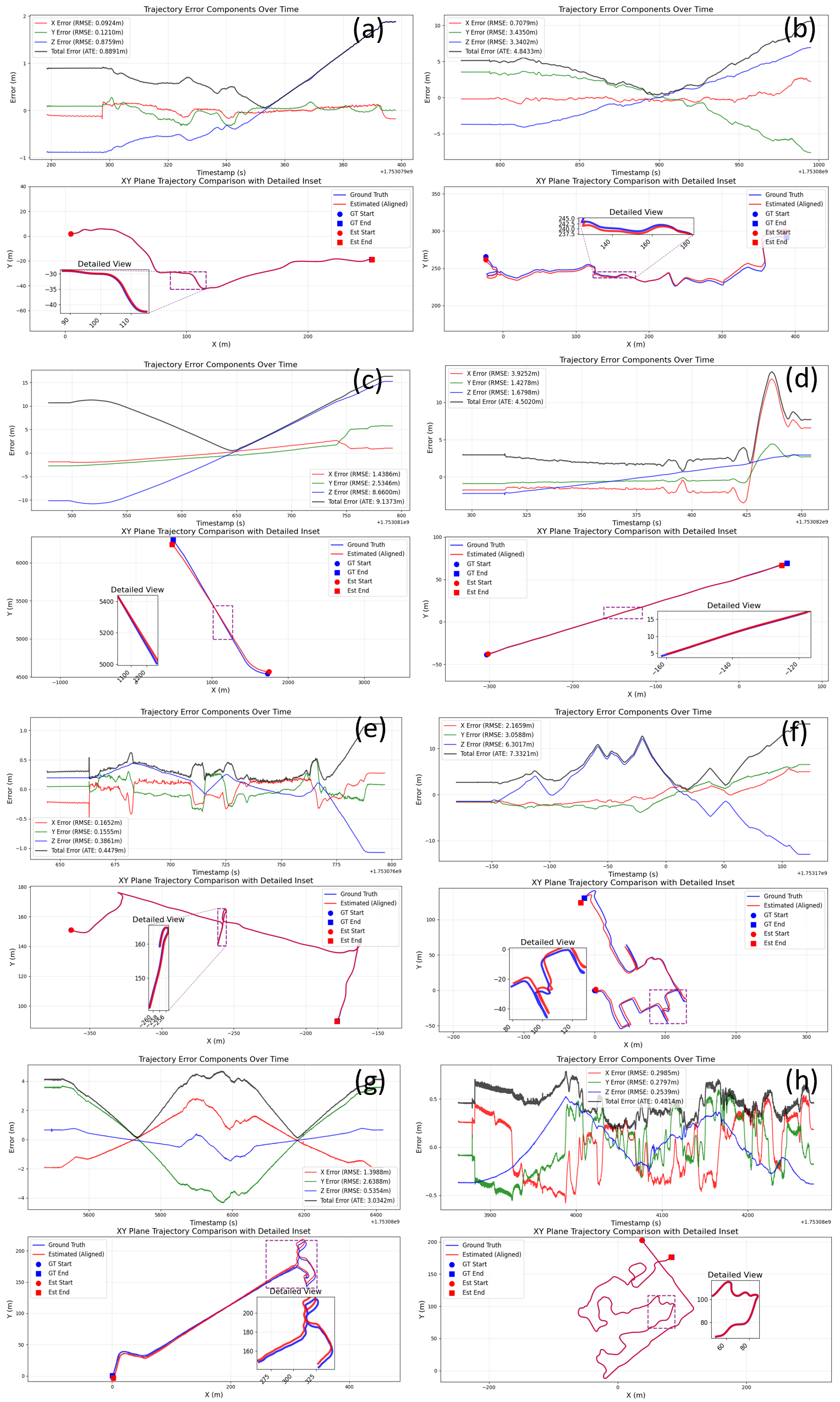}
    \caption{Simulation evaluations of the proposed AEOS. (a)-(g), the trajectories and ATE of the proposed AEOS for Simu-Seq01 to Simu-Seq08.}
    \label{fig:simu_evaluations}
\end{figure*}

\section{Experiments}\label{sec:experiments}

\subsection{Evaluation Metrics and Baselines}

To quantitatively evaluate the performance of the proposed AEOS policy, we adopt the Absolute Pose Error (APE) as the primary metric. APE measures the deviation between the estimated UAV trajectory and the ground-truth poses, providing a direct indicator of localization accuracy.
We compare AEOS against five representative baselines:

(1-2) \textbf{Fixed-Rate Scanning (Slow/Fast):} Two constant-speed motorized LiDAR configurations, rotating at 1 rad/s and 8 rad/s respectively. These emulate common non-adaptive scanning setups in UAV mapping.

(3) \textbf{Optimization-Based MPC:} A model predictive controller that relies solely on analytically computed uncertainty, without any learned exploration cost.

(4) \textbf{Random Control:} Angular velocities are randomly sampled from a uniform distribution within the control bounds, serving as a lower-bound reference.

(5) \textbf{AEOS w/o $\mathcal{C}_{\text{unc}}$}: A variant of AEOS with the uncertainty cost $\mathcal{C}_{\text{unc}}(\theta_k)$ disabled, isolating the contribution of the learned neural exploration component.

All methods are executed under identical simulation conditions, using the same trajectory references, LiDAR sensor model, and environment configurations. APE is computed for each scene independently, and average results are reported to assess overall generalization.

\subsection{Evaluation in Simulation}

Table~\ref{tab:sim_ape_results} reports the Absolute Pose Error (APE) across eight simulated sequences, benchmarking various scanning control strategies. Our proposed method, AEOS, achieves the best performance in all test environments, with an average APE of $3.83 m$, significantly lower than both classical baselines and optimization-based controllers.

Fixed-rate strategies suffer from a trade-off between stability and coverage: slow scanning results in limited observability, while fast scanning induces unstable matching, especially in geometrically constrained settings such as tunnels and caves. Random control performs even worse due to a lack of temporal consistency.

The optimization-based MPC baseline provides more adaptive behavior by solving trajectory-level objectives, yielding improved accuracy in most scenes. However, it lacks the learning-based generalization capability, especially in structurally diverse environments like Seq03 (Eurasia Tunnel), where AEOS achieves a two-times improvement.

Ablation analysis further validates the benefit of including uncertainty-awareness: AEOS without $\mathcal{C}_{\text{unc}}$ leads to consistent degradation across sequences, underscoring the role of active perception in navigating occluded or texture-sparse regions. Overall, the hybrid cost formulation and learning-based planning enable AEOS to robustly optimize viewpoint trajectories for localization fidelity across complex and varied scenes.

\begin{table*}[t]
    \centering
    \caption{APE (in meters) of different scanning control methods across simulated scenes. Lower is better. Bold indicates best performance. Underline indicates second-best.}
    \label{tab:sim_ape_results}
    \begin{adjustbox}{width=\linewidth}
    \begin{tabular}{lccccccc}
        \toprule
        \textbf{Sequence} & \textbf{Scene} & \textbf{Fixed-Rate (Slow)} & \textbf{Fixed-Rate (Fast)} & \textbf{Random Control} & \textbf{Optimization-Based MPC} & \textbf{AEOS w/o $\mathcal{C}_{\text{unc}}$} & \textbf{AEOS (Ours)} \\
        \midrule
        Simu-Seq01 & Lava Tube       & 2.43 & 2.01 & 1.87 & \underline{1.12} & 1.39 & \textbf{0.89} \\
        Simu-Seq02 & Cave            & 7.33 & 6.78 & 6.41 & 5.37 & \underline{5.09} & \textbf{4.84} \\
        Simu-Seq03 & Eurasia Tunnel  & 120.02 & 101.53 & 130.21 & \underline{20.83} & 27.12 & \textbf{9.13} \\
        Simu-Seq04 & Wuhan Tunnel    & 7.19 & 6.92 & 6.24 & \underline{4.89} & 5.12 & \textbf{4.50} \\
        Simu-Seq05 & Wuhan Subway    & 2.11 & 1.66 & 1.38 & 0.88 & \underline{0.63} & \textbf{0.45} \\
        Simu-Seq06 & Building        & 9.72 & 9.01 & 8.56 & \underline{7.88} & 8.01 & \textbf{7.33} \\
        Simu-Seq07 & Spine           & 5.02 & 4.53 & 4.02 & 3.28 & \underline{3.11} & \textbf{3.03} \\
        Simu-Seq08 & Forest          & 1.61 & 1.33 & 1.07 & \underline{0.62} & 0.74 & \textbf{0.48} \\
        \midrule
        \textbf{Average} & --   & 19.93 & 16.72 & 22.00 & \underline{5.73} & 6.39 & \textbf{3.83} \\
        \bottomrule
    \end{tabular}
    \end{adjustbox}
\end{table*}

\subsection{Evaluation in Real-world Environments}

\begin{figure*}[h]
    \centering
    \includegraphics[width=\linewidth]{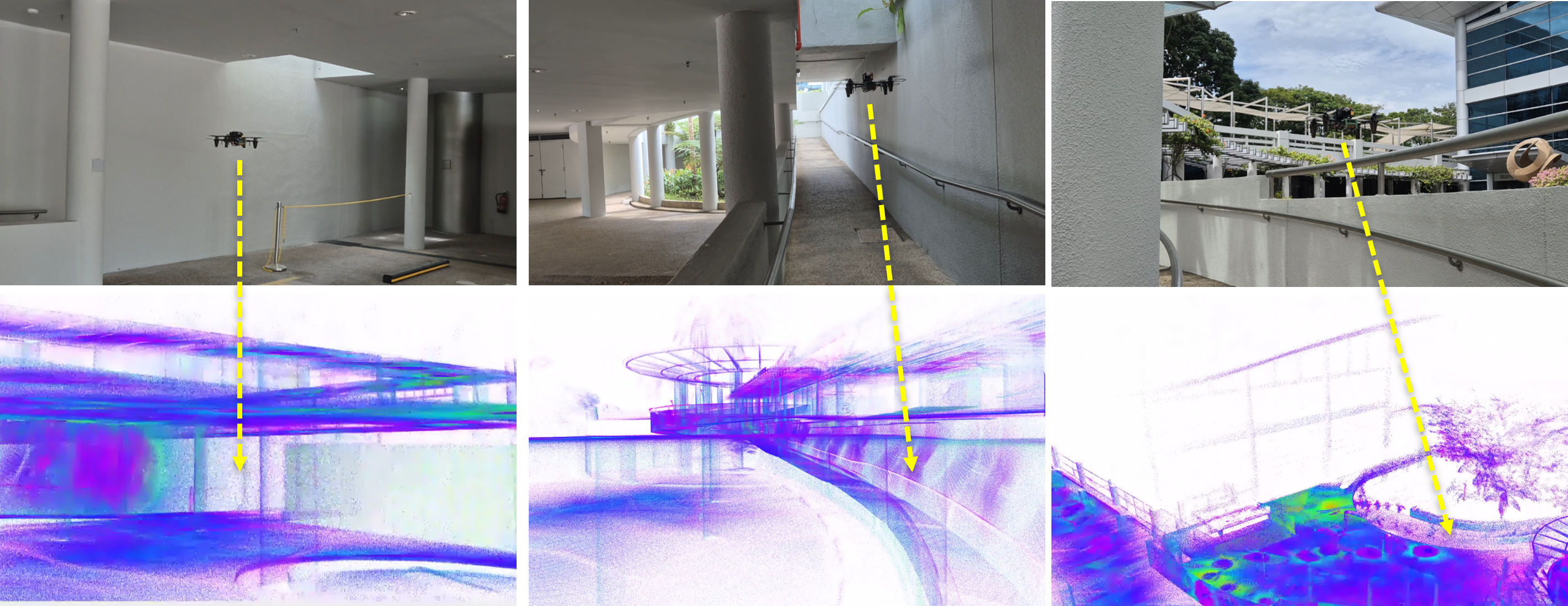}
    \caption{Real-world evaluation of the proposed AEOS in real-seq1.}
    \label{fig:real-seq1}
\end{figure*}

\begin{figure*}[h]
    \centering
    \includegraphics[width=\linewidth]{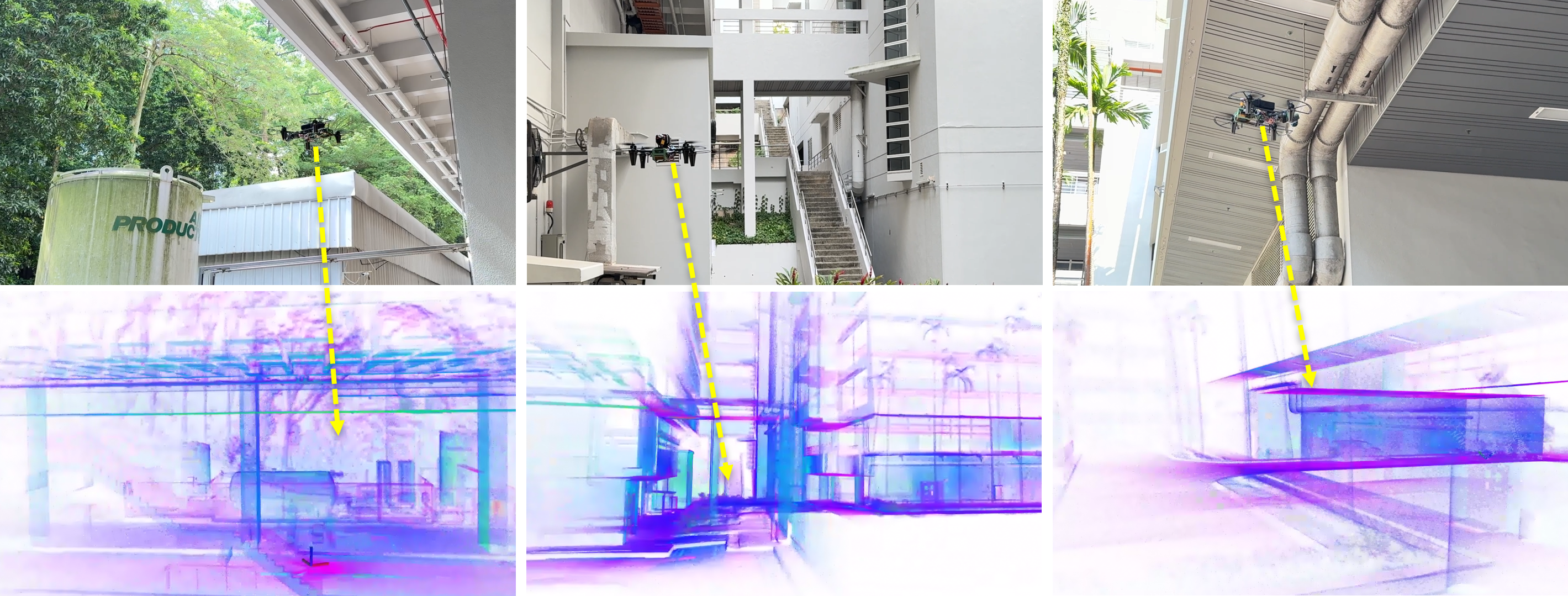}
    \caption{Real-world evaluation of the proposed AEOS in real-seq2.}
    \label{fig:real-seq2}
\end{figure*}

\begin{figure*}[]
    \centering
    \includegraphics[width=\linewidth]{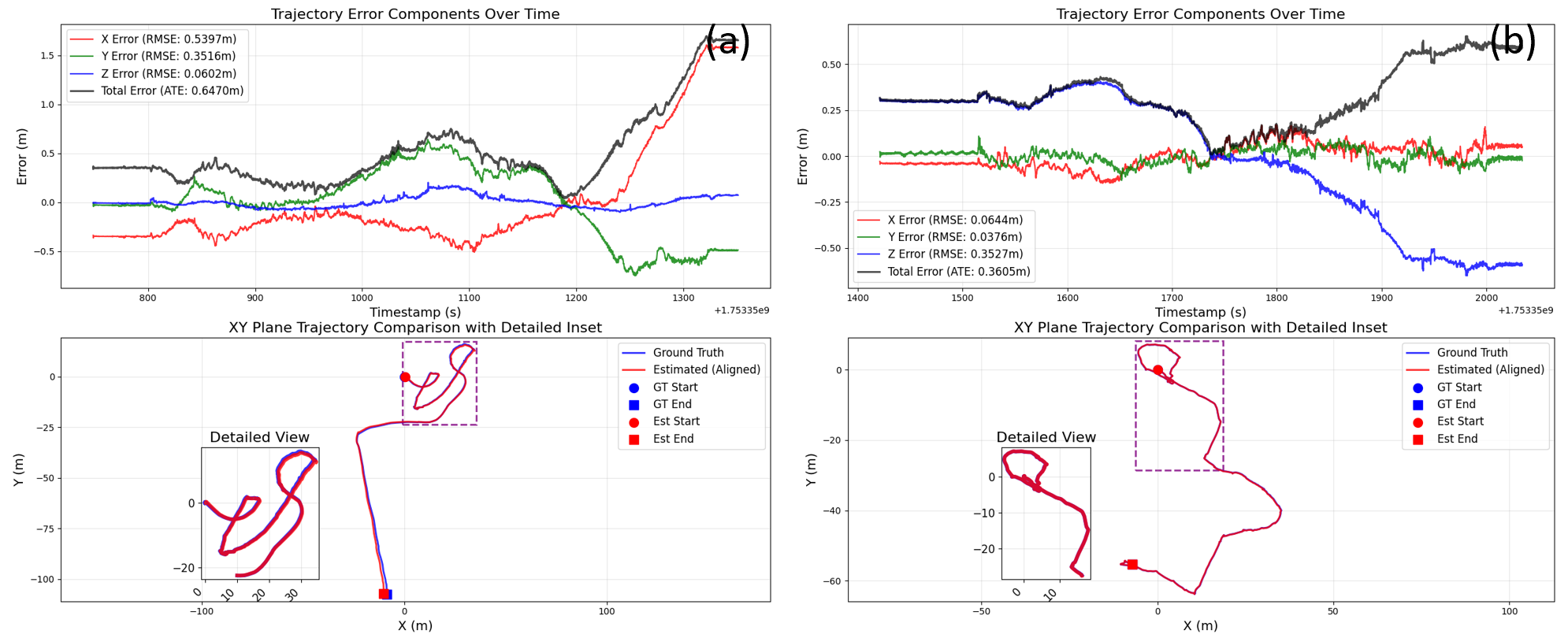}
    \caption{Real-world evaluation.(a) and (b) the trajectories and ATE of the proposed AEOS for Real-Seq01 to Real-Seq02.}
    \label{fig:real_evaluation}
\end{figure*}

\begin{table*}[t]
    \centering
    \caption{Description of the real-world evaluation.}
    \label{tab:real_dataset_stats}
    \begin{adjustbox}{width=\linewidth}
    \begin{tabular}{lcccccc}
        \toprule
        \textbf{Sequence} & \textbf{Scene} & \textbf{Location} & \textbf{Environment Type}& \textbf{Average Velocity (m/s)} & \textbf{Trajectory Time (s)} & \textbf{Trajectory Length (m)} \\
        \midrule
        Real-Seq01 & Industrial Equipment     & Singapore        & Infrastructure & 0.32 & 613 & 194\\
        Real-Seq02 & Basement          & Singapore       & Infrastructure  &  0.48 & 602 & 287   \\
        \bottomrule
    \end{tabular}
    \end{adjustbox}
\end{table*}

To assess the real-world performance of the proposed AEOS framework, we conduct comparative experiments against baseline methods in two challenging indoor environments as shown in Fig. \ref{fig:real-seq1} and Fig. \ref{fig:real-seq2}. To ensure both safety and experimental fairness, we first acquire a high-resolution prior map of each scene using a survey-grade LiDAR scanner. This prior map serves as a reference for accurate LiDAR-to-map registration, enabling robust pose estimation and preventing potential collisions caused by unstable localization during flight. Based on the reference map, we then generate a consistent set of waypoints for the UAV by human control. This ensures that all control strategies, namely, AEOS and baselines, are evaluated along identical trajectories, providing a controlled and fair comparison of scanning performance and localization accuracy. The ATE and trajectories of proposed AEOS are illustrated in Fig. \ref{fig:real_evaluation}. Then we evaluate the APE of different algorithms against the ground-truth trajectories in Table. \ref{tab:real_ape_results}.

\begin{table*}[t]
    \centering
    \caption{APE (in meters) of different scanning control methods across real-world environments. Lower is better. Bold indicates best performance. Underline indicates second-best.}
    \label{tab:real_ape_results}
    \begin{adjustbox}{width=\linewidth}
    \begin{tabular}{lccccccc}
        \toprule
        \textbf{Sequence} & \textbf{Scene} & \textbf{Fixed-Rate (Slow)} & \textbf{Fixed-Rate (Fast)} & \textbf{Random Control} & \textbf{Optimization-Based MPC} & \textbf{AEOS w/o $\mathcal{C}_{\text{unc}}$} & \textbf{AEOS (Ours)} \\
        \midrule
        Real-Seq01 & Industrial Equipment & 1.73 & 1.42 & 2.05 & \underline{0.61} & 0.68 & \textbf{0.36} \\
        Real-Seq02 & Basement             & 1.88 & 1.35 & 2.12 & \underline{0.73} & 0.79 & \textbf{0.65} \\
        \midrule
        \textbf{Average} & --              & 1.81 & 1.39 & 2.09 & \underline{0.67} & 0.74 & \textbf{0.51} \\
        \bottomrule
    \end{tabular}
    \end{adjustbox}
\end{table*}

The quantitative results are summarized in Table \ref{tab:real_ape_results}. In both test sequences, the proposed AEOS framework achieves the lowest absolute pose error (APE), consistently outperforming all baseline methods. In Real-Seq01, which involves large, unstructured industrial equipment with frequent occlusions and geometric degeneracy, AEOS achieves an APE of 0.36 m, surpassing the second-best method, Optimization-Based MPC (0.61 m), by 41\%. In Real-Seq02, conducted in a confined basement environment with narrow corridors and repeated occlusion, AEOS maintains strong performance with an APE of 0.65 m, again outperforming Optimization-Based MPC (0.73 m) and all other baselines.

The superior performance of AEOS across both environments highlights the effectiveness of its hybrid control design. While Optimization-Based MPC benefits from model-based planning with uncertainty reasoning, it lacks adaptability to local scene variations. In contrast, AEOS leverages a learned cost map to capture high-dimensional spatial cues for adaptive exploration, while still maintaining interpretability and task-awareness via analytical uncertainty modeling. This combination enables AEOS to make informed scanning decisions in both feature-rich and degenerate regions.

In addition, AEOS demonstrates strong generalization and consistency, achieving the lowest average APE (0.51 m) and the smallest performance variance across scenes. Other baselines, such as fixed-rate or random scanning, suffer from either limited coverage or unstable observability, resulting in significantly higher localization errors.
These findings confirm that AEOS enables accurate, robust, and efficient LiDAR-inertial odometry for UAVs operating in real-world environments with complex geometry, occlusions, and sensing challenges.

\subsection{Time Performance Analysis}
\begin{figure}[]
    \centering
    \includegraphics[width=\linewidth]{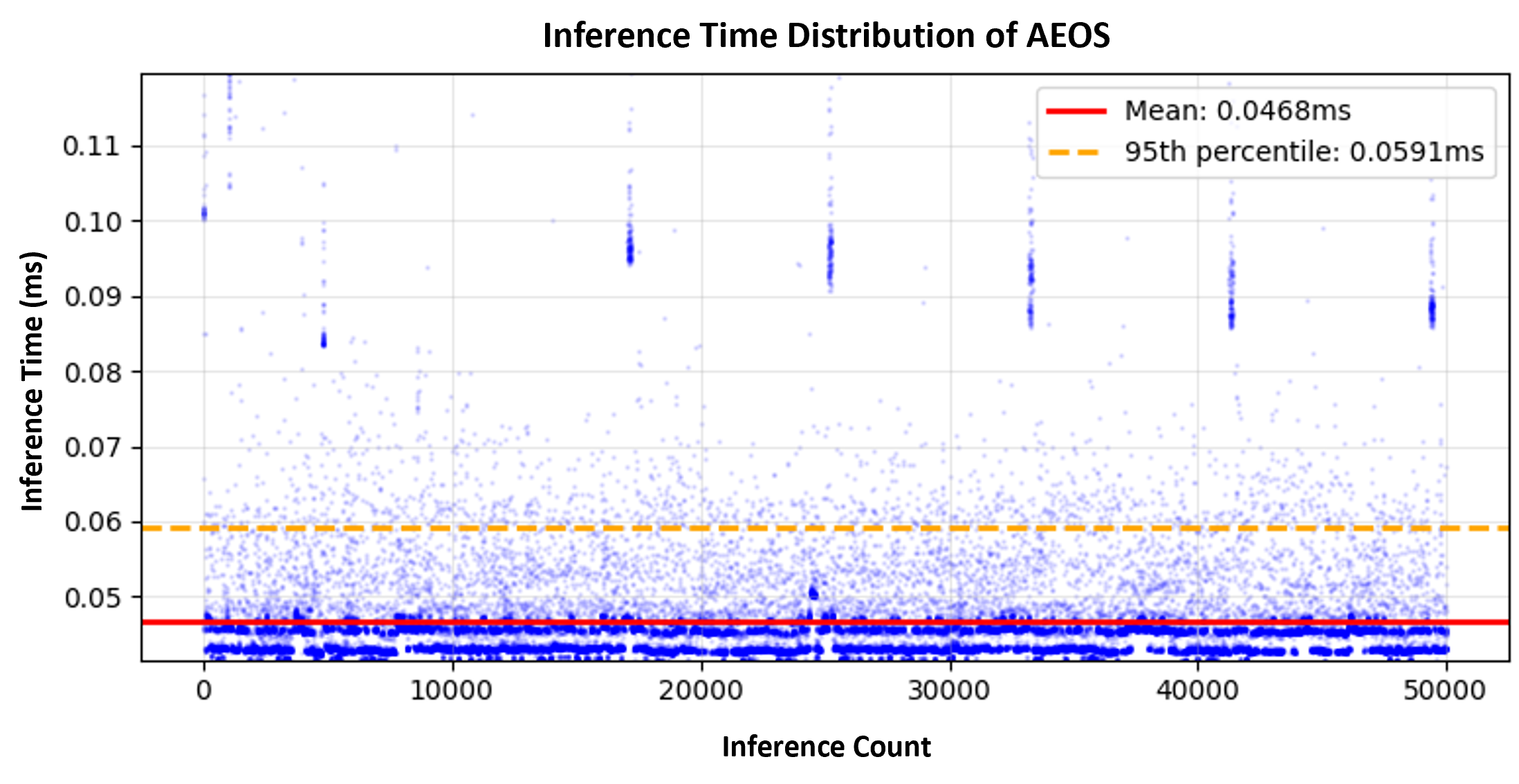}
    \caption{Inference time distribution of the proposed AEOS control policy evaluated on an x86 Intel-N305 edge computing unit. The results show that the system achieves real-time performance with low latency, suitable for onboard deployment.}
    \label{fig:time_distributions}
\end{figure}

Beyond localization accuracy, the practical deployability of AEOS also hinges on its computational efficiency for real-time onboard execution. To this end, we further evaluate the runtime performance of the control module on edge hardware.

We benchmarked the inference latency of AEOS on an x86 Intel-N305 edge computing unit. As shown in Fig. \ref{fig:time_distributions}, the average inference time is 0.047 ms, with 95\% of all inferences completed within 0.059 ms. These results confirm that AEOS meets the real-time requirements of UAV platforms, with negligible control latency. Moreover, the extremely low computational overhead ensures sufficient processing headroom for concurrent execution of other onboard tasks such as LiDAR-inertial odometry (LIO), mapping, and global planning, thereby facilitating seamless integration into embedded autonomous systems.

\section{Discussion and Future Work}\label{sec:diss}

\subsection{Interpretable Role Separation in Hybrid RL-MPC}

\begin{figure}[]
    \centering
    \includegraphics[width=\linewidth]{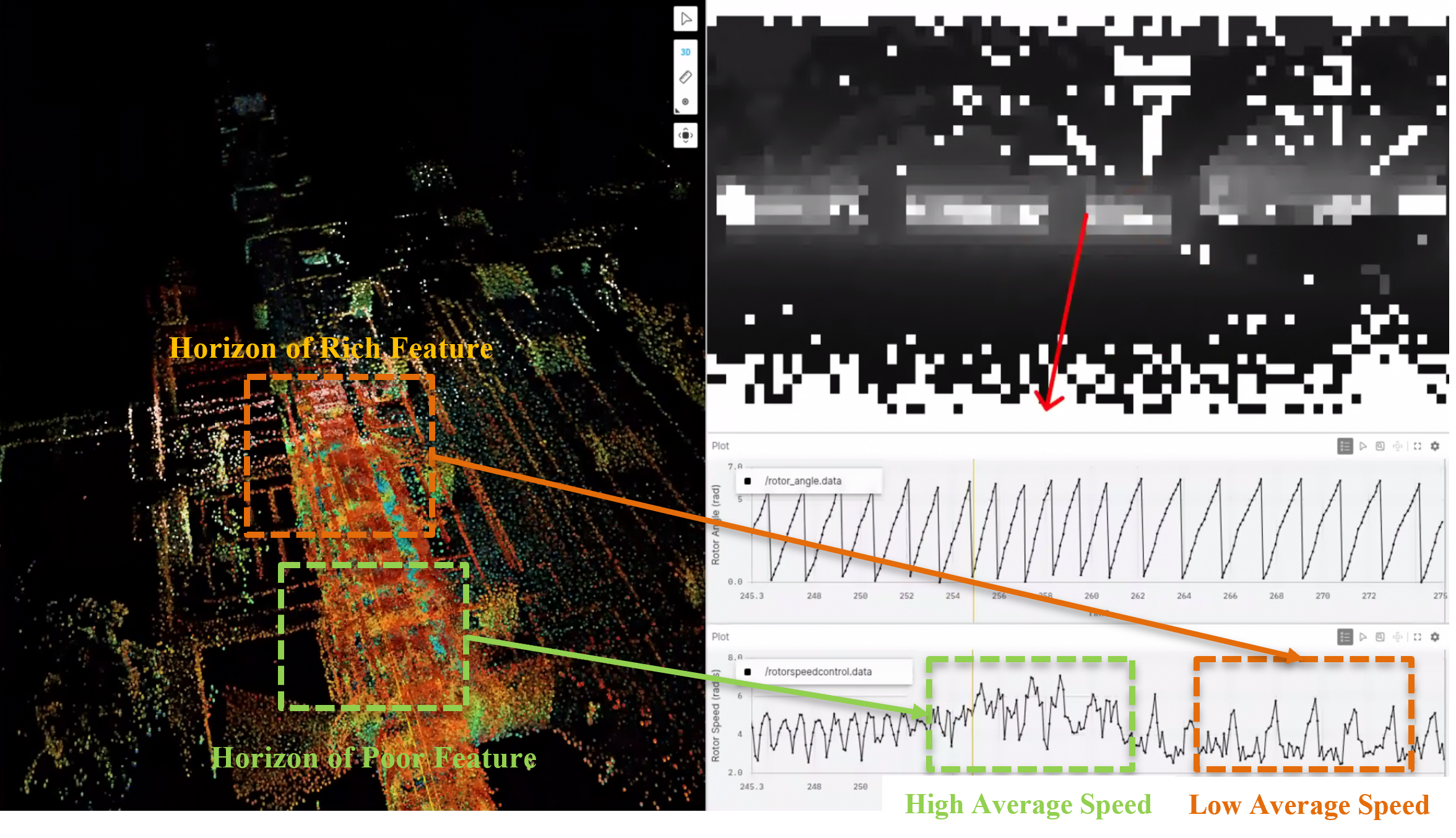}
    \caption{Temporal profile of the AEOS angular velocity output. The overall trend is guided by the RL-based learned cost (low-frequency), while high-frequency local variations are introduced by the uncertainty-aware MPC module.}
    \label{fig:interpretable}
\end{figure}

A key strength of the AEOS framework lies in its interpretable control behavior, emerging naturally from the hybrid architecture that combines reinforcement learning and model predictive control. Empirically, as illustrated in Fig.~\ref{fig:interpretable}, the angular velocity output over time exhibits a clear frequency-domain structure: a slowly varying trend reflecting long-term scanning intentions, and localized high-frequency variations that enable reactive adaptation to immediate sensing and localization needs.

This pattern corresponds to the functional separation of roles within AEOS. The learned neural cost map, trained via reinforcement learning, primarily governs the low-frequency envelope of the control signal. It modulates strategic behaviors such as whether to persist in scanning a direction or move quickly through sparse regions. In contrast, the analytical uncertainty-based MPC contributes to high-frequency corrections, adjusting the scan direction at finer time scales to improve local pose observability.

This separation is not only observed empirically but also rooted in the system's structural design. The input to the neural policy is a compressed panoramic depth map with limited spatial and temporal resolution. Such an input bandwidth inherently limits the ability of the policy to respond to fast state transitions or localized geometric detail, constraining its output to low-pass control signals that reflect coarse, scene-aware trends. Meanwhile, the MPC module operates with real-time state feedback and full-resolution geometric information, enabling precise, short-horizon refinements based on local observability.

Together, this layered control scheme achieves both adaptability and accuracy: the RL component provides scene-level scanning intent, while the MPC layer ensures responsive and precise adjustment. This interaction yields a semantically and physically interpretable behavior structure, which improves both transparency and robustness in complex real-world deployments.

\subsection{Application for Rescue in Complex Scenes}

\begin{figure}[]
\includegraphics[width=\linewidth]{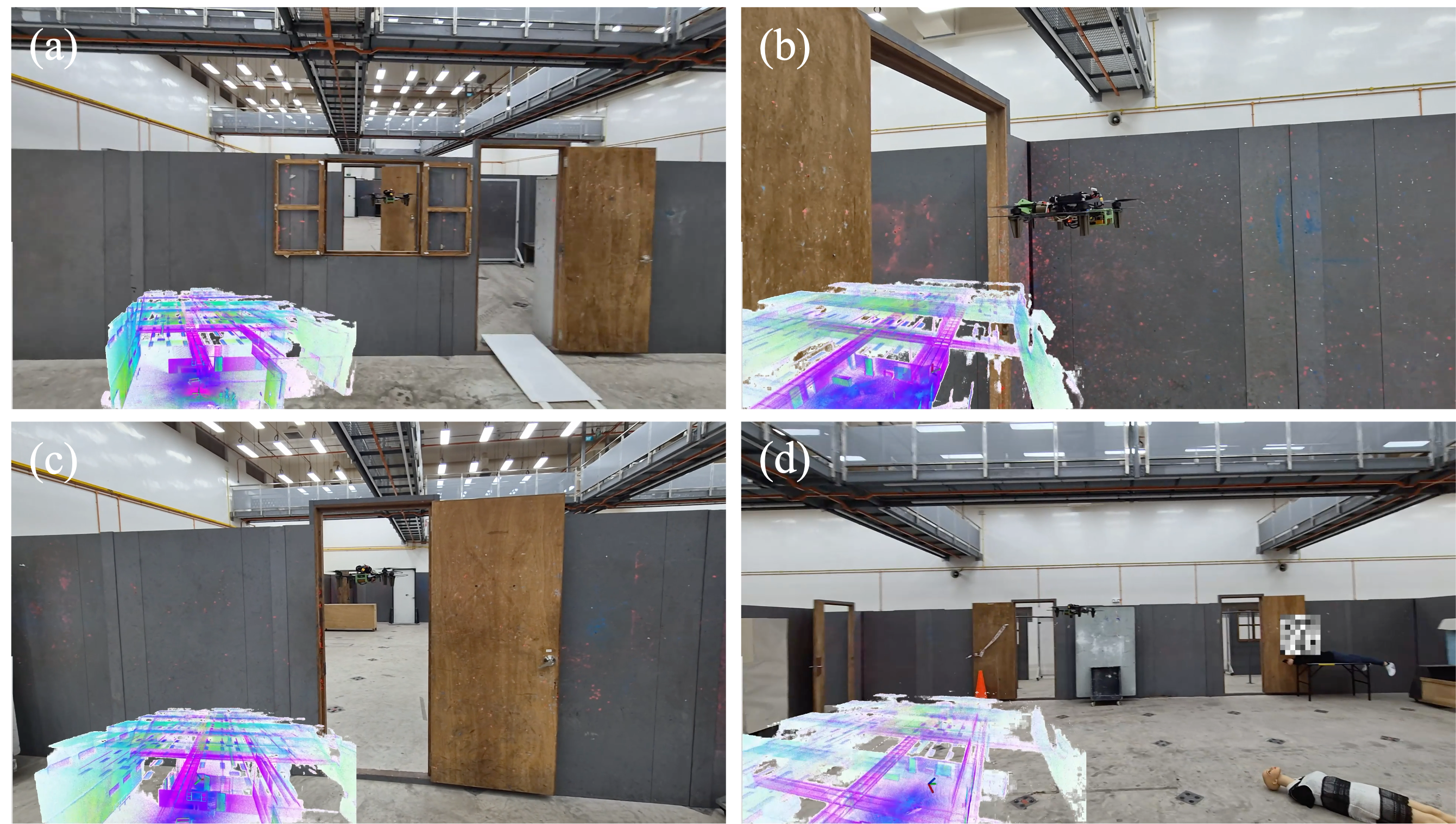}
    \captionof{figure}{Application for rescue in complex scenes. (a) Flying through the window. (b) Flying through the narrow corridor. (c) Flying through the door. (d) Quick exploration of the complex environment. The corresponding video can be found on the project page.}
    \label{fig:app_in_rescue}
\end{figure}

Benefiting from its reliable positioning and wide field of view (FoV), the proposed AEOS exhibits strong potential for rescue operations in complex environments, as illustrated in Fig.~\ref{fig:app_in_rescue}. The AEOS-drone is capable of safely traversing constrained spaces such as windows, doors, and narrow corridors, thereby demonstrating its robustness and applicability in extreme operational scenarios.

\subsection{Future Directions}
A promising future direction is to extend AEOS toward multi-agent active sensing scenarios, where multiple UAVs collaboratively perform LiDAR scanning and localization in large-scale or multi-level environments. By sharing observations and coordinating scanning directions, a team of agents can improve spatial coverage, reduce redundancy, and accelerate map acquisition. This requires the design of decentralized policies that scale with team size while respecting bandwidth, latency, and safety constraints. Incorporating inter-agent communication, coordination-aware reinforcement learning, and uncertainty-aware task allocation will be key to enabling efficient and robust multi-agent LiDAR-inertial odometry under real-world operational constraints.

\section{Conclusion}\label{sec:conclusion}

In this paper, we proposed AEOS, a biologically inspired and computationally efficient active scanning control framework for UAV-based LiDAR-inertial odometry in complex environments. AEOS integrates a lightweight motorized LiDAR system with a hybrid control strategy that fuses model predictive control (MPC) and reinforcement learning (RL). Specifically, an analytical uncertainty model guides task-aware exploitation by predicting future pose observability, while a learned neural cost map promotes exploration by capturing high-dimensional spatial cues from panoramic depth representations. This design enables scene-adaptive, interpretable, and low-latency control suitable for real-time onboard deployment on resource-constrained UAV platforms.

We developed a high-fidelity point cloud-based simulation environment to facilitate scalable training and sim-to-real transfer of the AEOS policy across diverse scenarios, including tunnels, forests, and underground infrastructure. Extensive experiments in both simulated and real-world environments demonstrate that AEOS consistently improves odometry accuracy and mapping completeness over fixed-speed, optimization-based, and purely learning-based baselines.
Future work will explore multi-agent coordination for collaborative LiDAR scanning using the proposed active scanning system.


\printcredits

\section{Declaration of interests}
The authors declare that they have no known competing financial interests or personal relationships that could have appeared to influence the work reported in this paper.

\section{Acknowledgment}
This research is supported by NTUitive Gap Fund (NGF-2025-17-006) and the National Research Foundation, Singapore, under its Medium-Sized Center for Advanced Robotics Technology Innovation (CARTIN).

\bibliographystyle{cas-model2-names}

\bibliography{refs}
\end{document}